%% file: main_arxiv_vfinal.tex
\documentclass[10pt,twocolumn,letterpaper]{article}

\usepackage{iccv}
\usepackage{times}
\usepackage{epsfig}
\usepackage{graphicx}
\usepackage{amsmath}
\usepackage{amssymb}

\usepackage[pagebackref=true,breaklinks=true,letterpaper=true,colorlinks,bookmarks=false]{hyperref}

\iccvfinalcopy % *** Uncomment this line for the final submission

\def\iccvPaperID{12198} % *** Enter the ICCV Paper ID here

\ificcvfinal\fi

\usepackage[accsupp]{axessibility} % Improves PDF readability for those with disabilities.

\usepackage{amsmath, amsthm, amssymb, hyperref, color}
\usepackage[shortlabels]{enumitem}
\usepackage{caption}
\usepackage{subcaption}
\usepackage{scalefnt}
\usepackage{verbatim}
\usepackage{tikz}
\usepackage{forest}
\usepackage{bm}
\usepackage{booktabs}
\usepackage{multirow}
\usepackage{soul}
\usepackage{thm-restate}
\usepackage{tikz-3dplot}

\usepackage{tabularx}
\usepackage{multirow}
\usepackage{diagbox}

\newcommand{\PreserveBackslash}[1]{\let\temp=\\#1\let\\=\temp}
\newcolumntype{C}[1]{>{\PreserveBackslash\centering}p{#1}}

\newtheorem{theorem}{Theorem}
\newtheorem{proposition}[theorem]{Proposition}
\newtheorem{lemma}[theorem]{Lemma}

\theoremstyle{definition}
\newtheorem{definition}[theorem]{Definition}

\newtheorem{example}[theorem]{Example}

\NewDocumentCommand{\pramudi}{ mO{} }{\textcolor{blue}{\textsuperscript{\textit{Pramu}}\textsf{\textbf{\small[#1]}}}}

\input{math_commands.tex}

\graphicspath{{figures/}} %Setting the graphicspath

\renewcommand{\ast}{{}^{\textstyle *}} % for raised "asterisks"

\begin{document}

%%%%%%%%% TITLE
\title{Linear Spaces of Meanings: \\
Compositional Structures in Vision-Language Models}

\author{Matthew Trager 
\and
Pramuditha Perera
\and
Luca Zancato\\
\and
Alessandro Achille
\and
Parminder Bhatia
\and 
Stefano Soatto\\
\and
% $^1$UCLA \qquad 
AWS AI Labs\\[.4cm]
{\tt\{mttrager,pramudi,aachille,parmib,soattos\}@amazon.com}\\
{\tt zancato@amazon.it}}
\maketitle
\ificcvfinal\thispagestyle{empty}\fi

%%%%%%%%% ABSTRACT
\begin{abstract}
We investigate compositional structures in data embeddings from pre-trained vision-language models~(VLMs). Traditionally, compositionality has been associated with algebraic operations on embeddings of words from a pre-existing vocabulary. In contrast, we seek to approximate representations from an encoder as combinations of a smaller set of vectors in the embedding space. These vectors can be seen as ``ideal words'' for generating concepts directly within embedding space of the model. We first present a framework for understanding compositional structures from a geometric perspective. We then explain what these compositional structures entail probabilistically in the case of VLM embeddings, providing intuitions for why they arise in practice. Finally, we empirically explore these structures in CLIP's embeddings and we evaluate their usefulness for solving different vision-language tasks such as classification, debiasing, and retrieval. Our results show that simple linear algebraic operations on embedding vectors can be used as compositional and interpretable methods for regulating the behavior of VLMs.
\end{abstract}

\vspace{-.4cm}

\section{Introduction}
\label{sec:intro}

In natural language, few primitive concepts or words can be used {compositionally} to generate a large number of complex meanings. For example, many composite concepts can be obtained by combining attributes and nouns. The hidden representations provided by a neural model, on the other hand, a priori \emph{do not} have a similar compositional structure. In contextual text embeddings, in particular, the representation of a string of text is jointly affected by all of its tokens simultaneously, which means that there may not be a simple relationship between the representations of the entire text and the words that appear in it. 

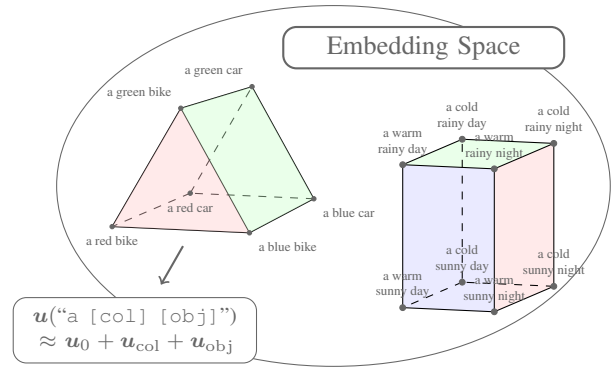
\begin{figure}[t]
    \centering
    \input{teaser_tikz.tex}
    \caption{\small Compositional structures in contextual embeddings. We show that the embeddings of composite concepts are often approximately decomposable as a sum of vectors corresponding to each factor. These vectors are not embeddings of actual words, but they can be viewed as ``ideal words'' and used for interpretable manipulations of the representations.}
    \label{fig:teaser}
    \vspace{-.4cm}
\end{figure}

In this paper, we investigate the existence of latent compositional structures in the embedding space. That is, we aim to decompose composite concepts as linear combinations of embedding vectors associated with different factors, as illustrated~in Figure~\ref{fig:teaser}.  If such vectors exist, they can be treated as \emph{ideal words} for composing new concepts directly within the representation space of the model. The first application that we envision is for vision-language models (\eg, CLIP~\cite{radfordLearningTransferableVisual2021}) where embeddings of text labels are often used for image classification or retrieval. In this setting, linear compositionality would imply that we could classify an image with $n_1 \ldots n_k$ composite labels---where $n_i$ indicates the number of options for each factor---by comparing each image with only $n_1 + \ldots + n_k$ ideal words, since by linearity the inner product of an image with a composed label is the sum of the product with the corresponding ideal words. Moreover, linear decompositions can be used for ``post-hoc'' manipulations of pre-trained data representations (\eg, amplifying or reducing the importance of certain factors), which can be helpful to control the behavior of neural models.

In general, the meaning of words in language is always \emph{contextual}, in the sense that their interpretation depends on any text that surrounds them. However, language would be completely impractical if words did not also have some stability in their meaning. The main benefit of the usage of words is, in fact, that meaning can be mostly inferred compositionally by combining meanings of words or phrases. There is, therefore, \emph{a natural tension between compositionality and contextuality}: the former requires some amount of independence from context, while the latter allows for general dependencies. In a sense, our goal in this work is to consider representations of meanings that were originally learned as contextual, and to later approximate them as needed with compositional ones based on ideal words. This combines the flexibility and expressiveness of contextuality with the structural efficiency of compositionality. Our main contributions can be summarized as follows:

\begin{itemize}
    \item We describe compositional linear structures from a geometric perspective  and explain how these structures can be approximately recovered from arbitrary collections of vectors associated with a product of ``factors.'' We also relate these structures with previous definitions of disentangled representations that were based on mathematical representation theory~\cite{higginsDefinitionDisentangledRepresentations2018a} (Section~\ref{sec:linear-factorizations}).
    \item We consider embeddings arising from visual-language models (VLMs) and show that the existence of decomposable embeddings is equivalent to 
    the conditional independence of the factors for the probability defined by the model. We also discuss some relaxations of this result that illustrate how linear structures may emerge even when if true data distribution satisfies weaker ``disentanglement'' conditions (Section~\ref{sec:vlm}).
    \item We empirically show that embeddings of composite concepts can often be well-approximated as linear compositional structures, and that this leads to simple but effective strategies for solving classification and retrieval problems in a compositional setting. We also visualize manipulations of decomposable embeddings using a CLIP-guided diffusion model (Stable Diffusion~\cite{Rombach_2022_CVPR}).
\end{itemize}

\section{Related Work}
\label{sec:rel_work}

Compositionality has long been recognized to be a
fundamental principle in cognition~\cite{fodor2002compositionality}. It has been a central in theme in Gestalt psychology~\cite{ellis2013source}, cognitive sciences~\cite{feldman1997regularity}, and pattern theory~\cite{geman2002composition}. The main benefit of compositional representations is that they avoid the combinatorial explosion that occurs if all composed concepts are considered to be completely distinct. This property is of course a characteristic feature of natural languages, which use a fixed vocabulary for all representions, making ``infinite use of finite means'' (von Humboldt)~\cite{chomsky2009syntactic}. However, while there is large body of work in NLP devoted to learning compositional representations of language (\eg,\cite{mitchellVectorbasedModelsSemantica, clarkVectorSpaceModels2015, baroniNounsAreVectorsa, fysheCompositionalInterpretableSemantic2015, coeckeMathematicalFoundationsCompositional}), modern text representations based on transformer architectures~\cite{vaswaniAttentionAllYou2017} are a priori \emph{not} compositional in any way. Some works have studied whether compositionality is implicitly present in neural networks, for example
by evaluating the ability of these models to generalize beyond the training data~\cite{hupkesCompositionalityDecomposedHow2020}. More relevant to our purposes,~\cite{andreasMeasuringCompositionalityRepresentation2019} proposed a framework for evaluating the compositionality of a network's internal representations, by searching for representational primitives; however, finding such compositional primitives requires solving an optimization problem.  
In a broad sense, compositionality can be seen as a particular way of exploiting or imposing \emph{structure} in the inner representations of a network. It has also been argued that data representations should be concentrated in low-dimensional linear spaces~\cite{maPrinciplesParsimonySelfConsistency2022, chanReduNetWhiteboxDeep2021}, or even be ``disentangled'' with respect to factors of variation in the data~\cite{higginsDefinitionDisentangledRepresentations2018a, burgessUnderstandingDisentanglingBeta2018, achilleEmergenceInvarianceDisentanglement2018}. Our perspective on compositional representations is closely related to the definition of disentanglement given in~\cite{higginsDefinitionDisentangledRepresentations2018a}.
As argued above, compositionality of text representations is naturally in tension with \emph{contextuality}. Since their introduction in NLP around 2018~\cite{petersDeepContextualizedWord2018, devlinBERTPretrainingDeep2019}, contextual text embeddings have been extremely successful, and are part of modern transformer-based architectures. 
The amount of contextuality in these word embeddings has been quantified using different metrics in~\cite{ethayarajhHowContextualAre2019a}. 

Linear compositionality for embeddings is often associated with popular ``vector analogies'' that are known to roughly hold for (non-contextual) word embeddings such as word2vec~\cite{mikolovEfficientEstimationWord2013} and GloVe~\cite{penningtonGloveGlobalVectors2014}. Several works have proposed theoretical justifications for this property~\cite{levyNeuralWordEmbedding2014, aroraLatentVariableModel2016, gittensSkipGramZipfUniform2017, allenAnalogiesExplainedUnderstanding,ethayarajhUnderstandingLinearWord2019,seonwooAdditiveCompositionalityWord2019}.
To our knowledge, however, similar properties for contextual embeddings of language models have not been considered, although~\cite{ushioBERTNLPWhat2021} has evaluated the performance of transformer-based models on analogy tasks.
Various limitations of linear analogies have also been pointed out~\cite{linzenIssuesEvaluatingSemantic2016, bouraouiRelationInductionWord}.

In the context of image generation, compositional approaches for controlling the output of diffusion models have been recently proposed in~\cite{liuCompositionalVisualGeneration2023,wangConceptAlgebraTextControlled2023}. In particular, \cite{wangConceptAlgebraTextControlled2023} introduced a ``concept agebra'' that is formally similar to our decomposable representations; however, their notion of ``concept'' is based on score representations (gradient of log-probabilities), rather than on embedding vectors, which leads to a different probabilistic characterization of compositionality.
Finally, \cite{chuangDebiasingVisionLanguageModels2023} introduced a method for removing biases and spurious correlations from pre-trained VLM embeddings for both discriminative and generative tasks; since their proposed approach consists in applying certain linear projections to textual embeddings (with some calibration adjustments), it can be seen as conceptually similar to an application of our decompositions. 

\vspace{.3cm}

\section{Decomposable Embeddings}
\label{sec:linear-factorizations}

We begin by discussing from a purely geometric perspective what we mean by ``linear compositionality.'' We consider a finite set $\mathcal Z = \mathcal Z_1 \times \ldots \times \mathcal Z_k$ that we view as representing a factored set of ``concepts.'' For example, the set $\mathcal Z$ may be a collection of strings of text organized in a structured way, \eg, according to attribute-object-context. We often write elements of $\mathcal Z$ as $z = (z_1,\ldots,z_k)$ with $z_i \in \mathcal Z_i$ and refer to $z_i$ as the components of $z$. We now consider an arbitrary embedding map $r: \mathcal Z \rightarrow V$ of $\mathcal Z$ into a vector space $V$.

\begin{definition}[Decomposable embeddings]\label{def:linearly_decomposable} 
A collection of vectors $r(\mathcal Z) = \{\vu_z \colon z \in \mathcal Z\} \subset V$ parameterized by $\mathcal Z = \mathcal Z_1 \times \ldots \times \mathcal Z_k$ is \emph{decomposable} if there exist vectors $\vu_{z_i} \in V$ for all $z_i \in \mathcal Z_i$ ($i=1,\ldots,k$) such that
\begin{equation}\label{eq:linear_factors_no_bias}
\vu_{z} = \vu_{z_1} + \ldots + \vu_{z_k},
\end{equation}
for all $z = (z_1,\ldots, z_k)$.
\end{definition}

This notion is very intuitive and can be seen as a generalization of the additive compositionality that has been considered for (pairwise) analogies and word embeddings~\cite{mikolovEfficientEstimationWord2013}. 

\vspace{.2cm}

\begin{restatable}{lemma}{basicproperties}\label{lem:basic-properties} 1) A collection of vectors $r(\mathcal Z)$ is decomposable if and only if the vector difference $\vu_{z} - \vu_{z'}$ does not depend on the components that $z,z' \in \mathcal Z$ share in common. 2) If $|\mathcal Z_i| = n_i$, then the dimension of $Span(r(\mathcal Z))$ is at most $1+ \sum_{i=1}^k (n_i-1)$.
\end{restatable}

It is easy to realize that if a collection of vectors $r(\mathcal Z)$ is decomposable, then the vectors appearing on the right of~\eqref{eq:linear_factors_no_bias} are \emph{never} uniquely determined. In particular, even though each $\vu_{z_i}$ is associated with a value of a factor $z_i \in \mathcal Z_i$, that vector cannot carry any ``semantic'' content. However, we can recover uniqueness in the components by simply turning to a ``centered'' decomposition.

\begin{restatable}[Centered decomposition]{lemma}{centereddecomp}\label{lem:centered_decomposition}
If a collection of vectors $r(\mathcal Z)$ is decomposable, then there exist unique vectors $\vu_0 \in V$ and $\vu_{z_i} \in V$  for all $z_i \in \mathcal Z_i$ ($i=1,\ldots,k$) such that $\sum_{z_i \in \mathcal Z_i} \vu_{z_i} = 0$ for all $i$ and
\begin{equation}\label{eq:linear_decomposition}
\vu_{z} = {\vu}_0 + \vu_{z_1} + \ldots + \vu_{z_k},
\end{equation}
for all $z = (z_1,\ldots, z_k)$.
\end{restatable}
In the previous decomposition, the vectors $\vu_{z_i}$ are now uniquely associated with the value of a factor $z_i \in \mathcal Z_i$, but are \emph{relative} to the other values in $\mathcal Z_i$ (since they sum to zero). Similarly, the vector spaces $V_{\mathcal Z_i} := Span(\vu_{z_i}\colon z_i \in \mathcal Z_i)$ are uniquely associated with each factor $\mathcal Z_i$. In our applications, we will refer to $\vu_i$ as the \emph{ideal words} of the linear factorization and to each $V_{\mathcal Z_i}$ as the \emph{semantic space} associated with $\mathcal Z_i$. Despite its simplicity, we believe that the decomposition in Lemma~\ref{lem:centered_decomposition} paints an interesting intuitive picture of linear models of ``meaning.'' In this setting, the origin is not a universally meaningful point; for example, the origin of text embeddings does not correspond to the null string. Thus, meanings might be best viewed as an \emph{affine space}, where the origin is only chosen as a particular reference that may depend on context. Ideal words, on the other hand, provide \emph{relative meanings} with respect to the context. 

From Lemma~\ref{lem:basic-properties}, it follows that decomposable representations must be very low-dimensional and, in particular, ``generic'' embeddings will \emph{not} be decomposable. However, it is very easy to recover the nearest decomposable approximation for any given set of vectors $\vu_z, z \in \mathcal Z$.

\begin{restatable}{proposition}{approximation}\label{prop:approximation} Let $\alpha_{z_i}$ $z_i \in \mathcal Z_i$ be arbitrary positive weights such that $\sum_{z_i \in \mathcal Z_i} \alpha_{z_i} = 1$, and define $\beta_{z} := \prod_{i} \alpha_{z_i}$  for all $z = (z_1,\ldots,z_k)$. Then, for any norm $\|\cdot \|$ induced by an inner product on $V$, we have that
\begin{equation}\label{eq:best-approx}
\begin{aligned}
&\arg \min_{\tilde \vu_z} \sum_{z \in \mathcal Z} \beta_z \|\vu_z - \tilde \vu_z\|^2,\\
&\qquad s.t. \,\, \{\tilde \vu_z\} \text{ is decomposable},
\end{aligned}
\end{equation}
is given by $\tilde \vu_z = {\vu}_0 + \vu_{z_1} + \ldots + \vu_{z_k}$ where
\begin{equation}\label{eq:projection-weighted}
\vu_0 := {\sum_z \beta_z \vu_z}, \,\,  \vu_{z_i} := \frac{1}{\alpha_{z_i}}\sum_{\substack{z' = (z_1',\ldots, z_k')\\z_i' = z_i}} \beta_z \vu_{z'} - \vu_0.
\end{equation}
\end{restatable}

This fact shows that computing decomposable approximations amounts to performing simple weighted averages of the original vectors. In many cases, we will consider $\alpha_{z_i} = \frac{1}{n_i}$ and $\beta_z = \prod \frac{1}{n_i}$, however it can be useful to allow for additional ``knobs,'' as the following example illustrates.

\vspace{.3cm}

\begin{example} \label{ex:basic-example}
One of our main motivations to consider decomposable structures is to approximate (pre-trained) contextual text embeddings to obtain representations that are \emph{interpretable} and \emph{compositional}. More concretely, assume that each factor $\mathcal Z_i$ represents a finite collection of strings and that the representation $r: \mathcal Z_1 \times \ldots \times \mathcal Z_k \rightarrow V$ is defined by concatenating strings and then embedding the result using a contextual language encoder. For a very simple example, consider
\[
\mathcal Z = \{\text{a blue, a red, a green}\} \times \{\text{bike, house}\},
\]
which leads to six possible strings and six distinct embedding vectors.
Using  Proposition~\ref{prop:approximation}, we can easily find a decomposable approximation $\vu_{(col,obj)} \approx \vu_{0} + \vu_{col} + \vu_{obj}$, where $\vu_{col}$ and $\vu_{obj}$ are the ideal words representing a particular object and color from $\mathcal Z$. As we will see, these vectors can be used for semantic manipulations of embeddings. Note that ideal words are not the same as the encodings of the original words or substrings. In fact, quite intuitively, the meaning of ideal word vectors is determined entirely by the way in which the corresponding string interacts with other factors. For example, we have $\vu_{\rm green} = \alpha_{car}\vu_{\rm (green\, car)} + \alpha_{house}\vu_{\rm (green\, house)} - \vu_0$ where $\vu_0$ is the mean of all six embeddings. In this particular example, ``green house'' has distinct contextual meaning, but this can be controlled by using appropriate weights, if desired. 
See Section~\ref{sec:exp} and Figure~\ref{fig:diffusion-colors} for more discussions on similar examples.
\end{example}

We conclude this section by pointing out a connection between decomposable embeddings and a notion of ``disentangled representations'' proposed in~\cite{higginsDefinitionDisentangledRepresentations2018a}. We refer to the Appendix for a short summary of the relevant mathematical background and for additional discussions. In a broad sense, we can say that an embedding map $r: \mathcal Z \rightarrow V$ into a vector space $V$ is ``linearly compositional'' with respect to some group of transformations $G$ if 1) $G$ acts on the set $\mathcal Z$ 2)~$G$ acts on $V$ as invertible linear transformations, and 3)~$r$ is a $G$-morphism, that is, if $r( g \cdot z ) = g \cdot r(z)$. In our case of interest, the set $\mathcal Z = \mathcal Z_1 \times \ldots \times \mathcal Z_k$ is a finite set of composite concepts (\eg, \{rainy, sunny\} $\times$ \{morning, evening\}) and $G = \mathfrak S_{n_1} \times \ldots \times \mathfrak S_{n_k}$ is a product of symmetric groups that acts on $\mathcal Z$ by varying each component separately (\eg, swapping ``rainy'' $\leftrightarrow$ ``sunny'' and ``morning'' $\leftrightarrow$ ``evening,'' independently). Following~\cite{higginsDefinitionDisentangledRepresentations2018a}, we say that the action of $G$ on $V$ is ``linearly disentangled'' if there exists a decomposition $V = V_1 \oplus \ldots \oplus V_k$ such that $g = (g_1 v_1, \ldots, g_k v_k)$ for all $v=(v_1,\ldots,v_k) \in V$ and $g = (g_1,\ldots,g_k) \in G$. Intuitively, this means that we can permute the different factors independently by acting with linear transformations on the embedding space. With these definitions in place we have that linear factorizations of embeddings are intimately related to disentangled compositional representations.

\begin{restatable}{proposition}{groupaction} 
Let $r(\mathcal Z)$ be a set of decomposable vectors of maximal dimension. Then $r$ is compositional for some disentangled action of $G = \mathfrak S_{n_1} \times \ldots \times \mathfrak S_{n_k}$ on $V$. Conversely, if $r$ is compositional for a disentangled action of $G$, then the vectors $r(\mathcal Z)$ are decomposable.
\end{restatable}

% Our ideal words can be viewed as a choice of affine (barycentric) coordinates for the affine space of meaning. A linearly factorization of a (subspace of) meaning means that the the vector space underlying the affine space is decomposed as a sum of affine spaces.

\section{Decomposable Embeddings in Vision-Language Models}
\label{sec:vlm}

In this section, we discuss linear factorizations from a probabilistic viewpoint in the context of vision-language models (VLMs). A priori, it may not be clear why the geometric notion of decomposable embeddings should be relevant in practice---for example, in the case of CLIP's normalized embeddings, it may seem that non-linear spherical geometry should come into play. In this section, however, we argue that vector factorizations have simple probabilistic intepretations, and in particular, we should expect these structures to be present in real data embeddings.

In the following, we write $\mathcal X$ for a set of texts and $\mathcal Y$ for a set of images (for simplicity, we consider a finite set of text and images, which will always be the case in practice). We consider a VLM that uses parametric encoders of texts $x \mapsto \vu_x$ and of images $y \mapsto \vv_y$ into $V = \RR^d$ to model the conditional log-probabilities of $x$ given $y$ and $y$ given $x$ in a bilinear fashion:
\begin{equation}\label{eq:bilinear_prob}
p(x \, |\, y) = \frac{\exp \vu_x^\top \vv_y}{\sum_{x'} \exp \vu_{x'}^\top \vv_{y}}, \quad p(y \, | \, x) = \frac{\exp \vu_x^\top \vv_y}{\sum_{y'} \exp \vu_{x}^\top \vv_{y'}}.
\end{equation}
For example, CLIP~\cite{radfordLearningTransferableVisual2021} uses both expressions in~\eqref{eq:bilinear_prob} to optimize a symmetric cross-entropy.
This setup is similar to the one used in NLP for context-based embeddings~\cite{mikolovEfficientEstimationWord2013} and also in transformer-based language modeling~\cite{vaswaniAttentionAllYou2017}, the main difference being that in those cases only one of the two expressions in~\eqref{eq:bilinear_prob} is used (to model words based on context). Much of the discussion that follows can be applied to these cases as well, but we focus on VLMs for clarity.

For any given pair of embeddings $\vu_x, \vu_y$ there exists a unique probability $p(x,y)$ on $\mathcal X \times \mathcal Y$ compatible with these embeddings which satisfies
\begin{equation}\label{eq:vlm}
 \log p(x,y) = \vu_x^\top \vv_y + c, \quad c \in \RR.
\end{equation}
In the following, we consider the distribution on $\mathcal X \times \mathcal Y$ expressed by a model and defined by~\eqref{eq:vlm}. After the learning stage, this distribution should reflect a ``true'' distribution on the same space. We remark, however, that the embedding dimension $d$ is in practice much smaller than the number of images or texts used in training, which means that we are actually imposing a \emph{low-rank constraint} on the joint probability distribution. In NLP, this effect has been referred to as the ``softmax bottleneck''~\cite{yangBreakingSoftmaxBottleneck2018}. 

We now consider a set of factors $\mathcal Z = \mathcal Z_1 \times \ldots \times \mathcal Z_k$ and assume that each $z \in \mathcal Z$ is represented by a string $x(z) \in \mathcal X$. Note that formally we could have associated factors with images rather than texts, however it is more natural to express discrete concepts as text. The factors can correspond to combinations of particular tokens (\eg, attributes and objects) but the association with strings could potentially be more complex (\eg, (``royal'', ``man'') $\mapsto$ ``king''). The VLM model now provides an embedding of $\mathcal Z$ via $z \mapsto \vu_{x(z)}$.

\begin{restatable}{proposition}{linearcomp}\label{prop:linear_composition} 
In the setting described above, and assuming that $Span(\vv_y, y \in \mathcal Y) = \RR^d$, the embedding $z \mapsto \vu_{x(z)}$ of $\mathcal Z$ is decomposable in the sense of Definition~\ref{def:linearly_decomposable} if and only if there exists functions $q_0, \ldots, q_k$ such that
\begin{equation}\label{eq:probabilistic_independence}
p(x(z), y) = q_0(y)q_1(z_1, y) \ldots q_k(z_k,y),
\end{equation}
for all $z = (z_1,\ldots,z_k) \in \mathcal Z$ and $y \in \mathcal Y$.
\end{restatable}

\begin{restatable}{corollary}{linearcompcor}\label{cor:probability-disentanglement} Under the assumptions of Proposition~\ref{prop:linear_composition}, an embedding $z \mapsto \vu_{x(z)}$ of $\mathcal Z$ is decomposable if only if the factors $z_i$ are conditionally independent given any image $y$. 
\end{restatable}

It is perhaps not surprising that the log-linear form of the model translates multiplicative decompositions into additive ones. It may be counterintuitive, however, that the conditional probabilities $p(z_i|y)$ as $y$ varies actually depend on \emph{all} of the ideal word vectors $\vu_{z_i}$, since normalizing constants can change with $y$. Indeed we have that
\begin{equation}\label{eq:single-factor}
\begin{aligned}
&p(z_i \, | \, y) = %\\
\exp(\vu_{z_i}^\top \vv_y)h(\mathcal Z_{j\ne i}, y),
\end{aligned}
\end{equation}
where
$h(\mathcal Z_{j\ne i}, y)$ is a function that depends on $y$ and all vectors corresponding to $\mathcal Z_j$ with $j \ne i$. In this sense, the geometric perspective of factorization is simpler since it disregards this dependence as $y$ varies.

The conditional independence from Proposition~\ref{prop:linear_composition} may seem like a strict requirement and may not be obviously true in the real world. For this reason, we discuss some relaxed conditions and explain what they imply in terms of decomposable structures.  First, given an image $y \in \mathcal Y$, we say that the probability $p(x(z), y)$ is \emph{mode-disentangled} (for the factor $\mathcal Z_i$) if 
\begin{equation}\label{eq:mode_disentanglement}
\argmax_{z_i\in \mathcal Z_i} p(x(z_i, z_{-i}),y) = \arg \max_{z_i \in \mathcal Z_i} p(x(z_i, z_{-i}'),y),
\end{equation}
for all $z_{-i}:=(z_1, \ldots, z_{i-1}, z_{i+1},\ldots, z_k)$ and $z_{-i}':=(z_1', \ldots, z_{i-1}', z_{i+1}',\ldots, z_k')$. Intuitively, this simply means means that it is possible to determine the most likely value of the factor $\mathcal Z_i$ by disregarding all of the remaining factors. Similarly, we say that $p(x(z), y)$ is \emph{order-disentangled} (for the factor $\mathcal Z_i$) if 
\begin{equation}\label{eq:order-disentaglement}
\begin{aligned}
&p(x(z_i, z_{-i}),y) \ge p(x(z_i', z_{-i}),y) \\
&\quad \Longleftrightarrow p(x(z_i, z_{-i}'),y) \ge p(x(z_i', z_{-i}'),y).
\end{aligned}
\end{equation}
for all $z_{-i}$ and $z_{-i}'$. This now means that it is possible to \emph{rank} the values of the factor $\mathcal Z_i$ by disregarding all of the remaining factors. It is easy to see that conditional independence implies order-disentanglement which in turn implies mode-disentanglement. If $|\mathcal Z_i| \le 2$, then mode-disentanglement and order-disentanglement are equivalent.

\begin{restatable}[Relaxed feasibility of linear factorizations]{proposition}{relaxation}\label{prop:relaxations} 1) If $y \in \mathcal Y$ is such that $p(x(z), y)$ is mode-disentangled, then one can replace the embedding vectors $\vu_{x(z)}$ with their decomposable approximations $\tilde \vu_{x(z)}$ from Proposition~\ref{prop:approximation} (for any choice of weights) and obtain the same prediction for $z$ given $y$; 2) If $p(x(z), y)$ is order-disentangled for all images $y$  sampled from a distribution with full support over the unit sphere, then the vectors $\vu_{x(z)}$ are necessarily decomposable.
\end{restatable}

The second part of this statement means that, roughly speaking, we should espect that imposing order-disentanglement for an increasing number of images would gradually lead to decomposable embeddings.

\begin{example} Let $\mathcal Z$ be of the form $\{o_1, o_2\} \times \{c_1, c_2\}$ (objects, contexts) and let $x(z)$ be the corresponding collection of strings (\eg, $x(o_i,c_j)=$``a photo of a [$o_i$] in [$c_j$]''). Then mode and order disentanglement are equivalent and mean that
\begin{equation}\label{eq:ex-disentanglement}
\begin{aligned}
&p(x(o_1, c_1)|y) > p(x(o_2, c_1)|y) \\[.1cm]
&\qquad \Leftrightarrow p(x(o_1, c_2)|y) > p(x(o_2, c_2)|y),\\[.3cm]
&p(x(o_1, c_1)|y) > p(x(o_1, c_2)|y) \\[.1cm]
&\qquad \Leftrightarrow p(x(o_2, c_1)|y) > p(x(o_2, c_2)|y).
\end{aligned}
\end{equation}
These are reasonable conditions on the probability $p(x(z),y)$ since it is normally possible to discriminate object and context in an image independently. If $p(x(z),y)$ and $y$ satisfy~\eqref{eq:ex-disentanglement}, then the first part of Proposition~\ref{prop:relaxations} means that we can use two (approximate) ``ideal word'' vectors $\vu_{o_1} = -\vu_{o_2}$ and $\vu_{c_1} = -\vu_{c_2}$ instead of the four original vectors $\vu_{x(o_i,c_j)}$ to assign the correct label to $y$. The second part of Proposition~\ref{prop:relaxations} means that if~\eqref{eq:ex-disentanglement} holds for ``all'' images $y$ (\ie, vectors covering the unit sphere), then the original vectors $\vu_{x(o_i,c_j)}$ are actually decomposable.
\end{example}

\section{Experiments}
\label{sec:exp}

We now empirically investigate the presence and usefulness of decomposable structures in real VLM embeddings. In all of our experiments, we use a pre-trained CLIP encoder \cite{radfordLearningTransferableVisual2021}\footnote{We use the HuggingFace implementation of CLIP with the publicly available checkpoint based on a ViT-L/14 vision transformer. See \url{https://huggingface.co/openai/clip-vit-large-patch14}}. Unless stated otherwise, we compute decomposable approximations of embeddings using Proposition~\ref{prop:approximation} with $\alpha_{z_i} = \frac{1}{n_i}$ and $\beta_z = \prod \frac{1}{n_i}$. We use different datasets that have a compositional nature: MIT-states~\cite{isolaDiscoveringStatesTransformations2015} and UTZappos~\cite{yuFineGrainedVisualComparisons2014}, that are image classification datasets where labels are pairs attribute--object; CelebA~\cite{liuDeepLearningFace2015} and Waterbirds~\cite{sagawaDistributionallyRobustNeural2020a} in which images have a label and a spurious attribute; and DeepFashion2~\cite{DeepFashion2} with PerVL annotations from~\cite{cohenThisMyUnicorn2022a}, where the goal is to retrieve object instances from different contexts. We also include a visualization of ideal words using a CLIP-guided diffusion model (Stable Diffusion 2.1\footnote{\url{https://huggingface.co/stabilityai/stable-diffusion-2-1}})~\cite{rombachHighResolutionImageSynthesis2022}. We emphasize that our goal is not to achieve state-of-the-art results, although we will see that linear manipulations can be surprisingly effective and sometimes outperform significantly more complex methods. Rather, we aim to show that linear decomposable structures in embedding spaces provide a useful conceptual and practical framework for \textit{understanding} and \textit{controlling} the behavior of pre-trained VLMs.

\begin{figure}
    \centering
    \includegraphics[width=.98\columnwidth]{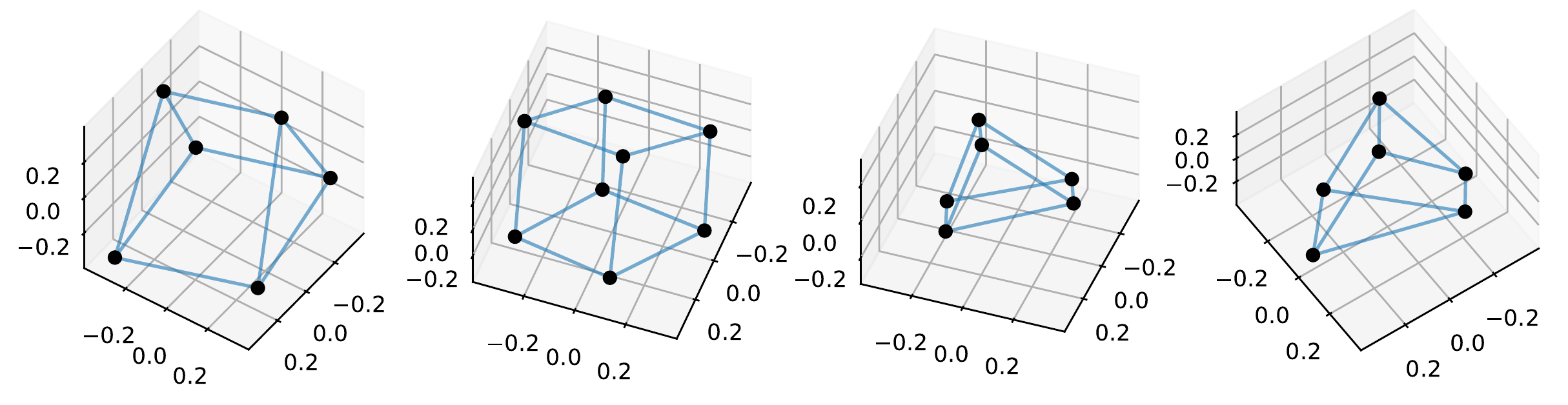}\\[.25cm]
    \includegraphics[width=.98\columnwidth]{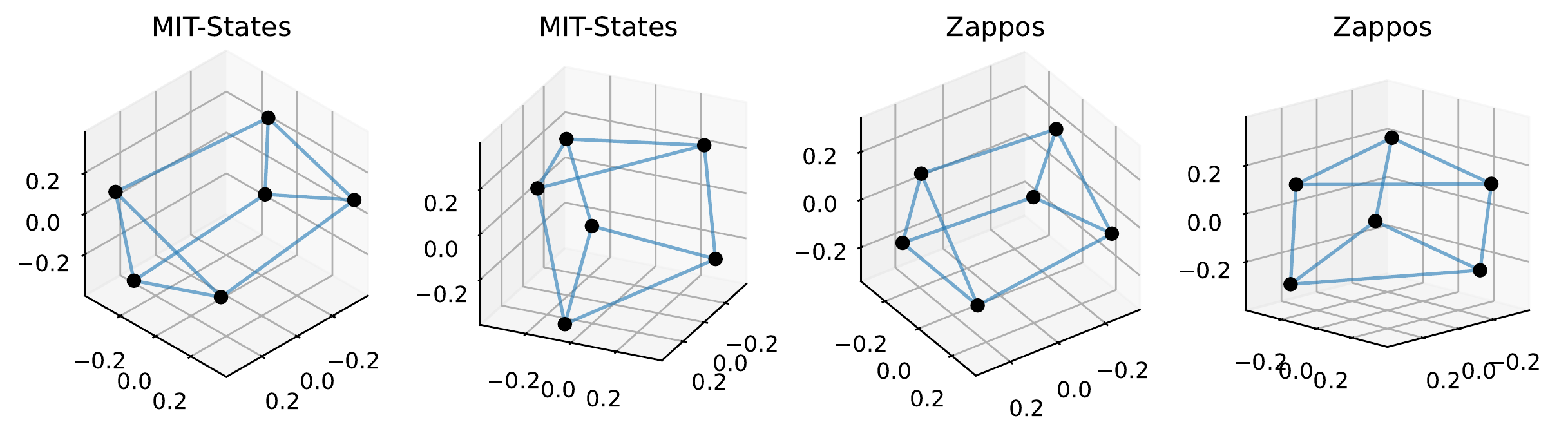}\\
    \caption{\textbf{Visualization of embeddings.} \emph{Top}: projected embeddings of manually constructed strings associated with decomposable concepts. \emph{Bottom:} projected embeddings for strings of the type ``an image of a [a] [o]'' for randomly chosen attributes and objects from MIT-states~\cite{isolaDiscoveringStatesTransformations2015} and UTZappos~\cite{yuFineGrainedVisualComparisons2014}. Symmetric structures indicate that embeddings are approximately decomposable. See text for details.}
    \label{fig:3dplots}
    \vspace{-.5cm}
\end{figure}

\vspace{-.3cm}

\paragraph{Visualization of embeddings.} 
Figure~\ref{fig:3dplots} shows some examples of embeddings of composite strings, visualized in 3D using PCA. In the top row, we show examples of manually constructed strings. In order: ``a photo of a \{red, blue, pink\} $\times$ \{car, house\}''; ``a photo of a \{big, small\} $\times$ \{cat, dog\} $\times$ \{eating, drinking\}''; ``\{a photo of a, a picture of a\} $\times$ \{place, object, person\}''; ``king, queen, man, woman, boy, girl'' (where one factor would correspond to male-female and the other to a generic context). In the bottom row, we present strings of the type ``an image of a [a] [o]'' for randomly chosen attributes and objects from MIT-states~\cite{isolaDiscoveringStatesTransformations2015} and UTZappos~\cite{yuFineGrainedVisualComparisons2014} (first using two attributes and three objects, and then using three attributes and two objects). Here we always use either $2\times 3$ or $2 \times 2 \times 2$ concepts since these decomposable structures have expected affine dimension 4, or linear dimension 3. The presence of roughly parallel edges and faces in these figures indicate that embeddings are approximately decomposable.
We note that in many of these examples the factorization of the concepts is already reflected in the \emph{syntax} of the strings, \ie, in the presence of repeated substrings in prompts with similar meaning. However, factorized vectors also encode semantic aspects, as can be seen in the last two examples from the first row. In the fourth example, the encoded strings have no repeated substrings, so the structure is ``emergent''; in the third example, the factor corresponding to \{a photo of a, a picture of a\} results in an ideal word vector with a smaller norm compared to the to other directions (resulting in a ``squashed'' triangular prism), as one might expect since this factor is not semantically significant. 
We refer to the Appendix for a more in-depth discussion. 

\vspace{-.3cm} 
\paragraph{Compositional classification.} We evaluate the usefulness of linear decomposable approximations for object-attribute labels 
of the MIT-states~\cite{isolaDiscoveringStatesTransformations2015} and UTZappos~\cite{yuFineGrainedVisualComparisons2014} datasets. The default strategy for applying CLIP in a zero-shot fashion on these datasets is to use text captions such as $x(a, o)$=``an image of a [$a$] [$o$].'' This results in $n_{obj} \times n_{attr}$ captions that each image must be compared with. We want to explore whether the embedding vectors $\vu_{x(a,o)}$ can be approximated with a decomposable set $\tilde \vu_{x(a,o)} = \vu_0 + \vu_a + \vu_o$, so that inference can be performed using only $n_{obj} + n_{attr}$ embedding vectors. The intuitive choice for such vectors would be to use the representations of captions such as ``image of a [$a$] object'' and ``image of a [$o$].'' We compare this choice with using the ``ideal words'' associated with the original captions, where the representation of an object $o$ is simply given by $\vu_o := \frac{1}{n_{attr}} \sum_a \vu_{x(a,o)}$, and similarly for attributes, as in Proposition~\ref{prop:approximation} (in this setting, there is no need to remove the mean vector $\vu_0$ since it is multiplied with every image vector). The resulting disjoint representations for objects and attributes ($\vu_o$ and $\vu_a$) are ``contextualized,'' in the sense that they optimally approximate the original pairwise embeddings. In Table~\ref{tab:ideal_comp_datasets}, ``pair'' refers to using the original pairwise labels, ``real words'' uses the embeddings of words corresponding to objects and attributes using ``image of a [$a$] object'' and ``image of a [$o$].'', while ``ideal words'' computes the vector ideal words for the factorization. We see that ideal words clearly outperform the \textit{real words} baseline, and often even surpass the accuracy of \emph{pair}. For MIT-States, using decomposable labels translates into using 360 vs. 28175 class vectors.
% (78$\times$ gain) and inference on 12995 test samples goes from 614ms to 9ms (63$\times$gain).

\begin{table}[]
\scriptsize
\centering
\begin{tabular}{@{}llccc@{}}
\toprule
                                     & \textbf{Method}             & \textbf{Pair Acc} & \textbf{Attr Acc} & \textbf{Obj Acc} \\ \midrule
% \mc{5}{c}{CLIP ViT-L/14} \\ \midrule
\multirow{3}{*}{{MIT-states}~\cite{isolaDiscoveringStatesTransformations2015}} & pair                & 7.7\%              & 16.2\%             & 47.8\%            \\ \cmidrule(l){2-5} 
                                     & real words            & 10.0\%             & 19.3\%             & 49.3\%            \\
                                    & ideal words & \textbf{11.5\%}             & \textbf{21.4\%}             & \textbf{50.8\%}            \\
                                     \midrule
\multirow{3}{*}{{UT Zappos}~\cite{yuFineGrainedVisualComparisons2014}}  & pair                & \textbf{12.4\%}             & 17.1\%             & \textbf{55.7\%}            \\ \cmidrule(l){2-5} 
                                     & real words            & 8.4\%              & 10.3\%             & 51.0\%            \\
                                     & ideal words & 10.8\%             & \textbf{19.2}\%             & 55.3\%            \\ 
                                     \bottomrule
\end{tabular}
\caption{\textbf{Zero-shot image classification results on compositional datasets.} Here ``pair'' refers to using all attribute-object pairs as candidate labels; ``real words'' refers to using labels corresponding to real words (\ie, separate attribute and object labels); ``ideal words'' refers to using compositional labels based on ideal words. Ideal words always lead to better accuracy than real words and often even outperform pairwise labels.}
\vspace{-.2cm}
\label{tab:ideal_comp_datasets}
\end{table}

\vspace{-.4cm}

\paragraph{Debiasing.} We can apply the decomposition into ideal words as a baseline strategy to remove contexts or biases from embeddings. The debiasing task can be formalized using the group robustness framework proposed in~\cite{sagawaDistributionallyRobustNeural2020a}. In this setting, we are given a collection of labels $\mathcal Y$ and spurious attributes $\mathcal A$, and we define a ``group'' as a pair $g \in \mathcal Y \times \mathcal A$. Assuming that each group corresponds to a probability $P_g$ on an input space $\mathcal X$, the goal is to find a classifier $f: \mathcal X \rightarrow \mathcal Y$ that leads to a small gap between worst-group error and average error:
\begin{equation}
\max_{g} \mathbb E_{x \sim P_g} \ell(f(x),y) - \mathbb E_{x \sim P}
 \ell(f(x),y)).
 \end{equation}
In a zero-shot setting with CLIP, classifiers are prompts that inherit biases from the dataset used in pre-training, so group robustness is not guaranteed. 
To address this problem, the authors of~\cite{chuangDebiasingVisionLanguageModels2023} propose a method for debiasing prompts that finds a projection map that makes spurious prompts irrelevant  (following~\cite{bolukbasiManComputerProgrammer2016b}) and then additionally regularizes the projection map to ensure that certain prompts are mapped near each other in embedding space. 
Here we note that a much simpler baseline would be to use ideal words to leverage the joint label-attribute representation provided by the pre-trained VL model and ``average out'' spurious attributes. More precisely, starting from a set of embeddings $\vu_{(y,a)}$ corresponding to prompts representing each group $g=(y,a)$, ideal words suggest to define the encoding of each label $y$ to be $\vu_{y} := \frac{1}{|\mathcal A|}\sum_{a \in \mathcal A} \vu_{(y, a)}.$
Once again, this is the same as the (shifted) ideal word corresponding to $y$, obtained by approximating pairwise embeddings of labels and attributes in a decomposable way.
Following~\cite{chuangDebiasingVisionLanguageModels2023}, we evaluate group robustness of unbiased prompts on the Waterbird~\cite{sagawaDistributionallyRobustNeural2020a} and CelebA~\cite{liuDeepLearningFace2015} datasets. For the Waterbird dataset, the labels are ``landbird'' and ``waterbird,'' and the confounding factor is water/land background. For the CelebA dataset, the labels are ``blond'' and ``dark'' hair and the confounding factor is the binary gender. For our simple unbiasing method, we prepend prompts associated with labels with prompts associated with spurious attributes, and then average over all the spurious prompts. In both datasets, we consider exactly the same prompts for spurious attributes and labels used in~\cite{chuangDebiasingVisionLanguageModels2023} (see the Appendix for a description). Our results are shown in Table~\ref{tab:group_robust}. On the CelebA dataset, our simple averaging strategy achieves a much smaller gap between average and worst group accuracy than the method proposed in~\cite{chuangDebiasingVisionLanguageModels2023} (1.6 vs 10.1). For Waterbird datsets, the gap is larger but comparable, and average accuracy is higher.
 
 \begin{table}[t]
\footnotesize
\centering
\begin{tabularx}{.45\textwidth}{l | c c c | c c c }
\toprule
&  \multicolumn{3}{c}{\textbf{Waterbird}~\cite{sagawaDistributionallyRobustNeural2020a}} & \multicolumn{3}{c}{\textbf{CelebA}~\cite{liuDeepLearningFace2015}}  \\
 & WG & Avg & Gap & WG & Avg & Gap  \\
\midrule
Zero-shot  & 45.3 & 84.4 & 39.1  & 72.8 & \textbf{87.6}  & 14.9
\\
Orth-Proj~\cite{chuangDebiasingVisionLanguageModels2023} & 61.4 & {86.4} & 25.0 & 71.1 & 87.0 & 15.9
\\
Orth-Cali~\cite{chuangDebiasingVisionLanguageModels2023} & \textbf{68.8} & 84.5 & \textbf{15.7} & {76.1} & 86.2 & 10.1
\\
Ideal Words & 64.6  & \textbf{88.0} &   23.3  & \textbf{83.9} & {85.5}  & \textbf{1.6} \\
\bottomrule
\end{tabularx}
\caption{\textbf{Group robustness results.} Ideal words can be used as a simple yet performant baseline for debiasing applications. }
\label{tab:group_robust}
\vspace{-.3cm}
\end{table}

\begin{table}[t]
\centering
    \resizebox{\columnwidth}{!}{
        \begin{tabular}{l|cccc}
        \toprule
            {} &    Text Only &  AvgImg+Text & PALAVRA~\cite{cohenThisMyUnicorn2022a} & IW \\
        \midrule
        DeepFashion2~\cite{DeepFashion2} & 17.6 $\pm$ 0.0 & 21.7 $\pm$ 2.4 & 28.4 $\pm$ 0.7$\ast$ & \textbf{37.0} $\pm$ 1.1 \\
        \midrule
        \midrule
        {} &    IW w.o.~mean removal &  \multicolumn{2}{c}{IW with Norm on mean} & IW \\
        \midrule
        DeepFashion2~\cite{DeepFashion2} & 22.1 $\pm$ 2.4 &  \multicolumn{2}{c}{36.5 $\pm$ 1.4} & \textbf{37.0} $\pm$ 1.1 \\
        \bottomrule
        \end{tabular}
    }
    \caption{\textbf{Concept retrieval results.} Mean Reciprocal Rank retrieval metric on the DeepFashion2~\cite{DeepFashion2} with annotations from PerVL~\cite{cohenThisMyUnicorn2022a}. Numbers with $\ast$ are taken from~\cite{cohenThisMyUnicorn2022a}.}
    \label{tab:main_retrieval}
\end{table}

\paragraph{Composing concepts and contexts.} We perform experiments using the DeepFashion2 dataset~\cite{DeepFashion2} with the captions provided in PerVL~\cite{cohenThisMyUnicorn2022a}. This dataset contains images of 100 unique fashion items (``concepts'') with textual descriptions. The task is to retrieve an image given a text query that includes a personalized concept that is specified using a small number of examples (5 samples). An example of a text query is ``The [CONCEPT] is facing a glass store display.''
In~\cite{cohenThisMyUnicorn2022a}, the authors propose a method called PALAVRA that trains new CLIP tokens to be associated with the custom concept; the learned tokens can then be used within natural language for retrieving images. The authors compare their method with a baseline approach dubbed ``AvgIm+Text'' which consists in averaging the CLIP embedding of the concept support images and of the embedded text query. This strategy is presented as the second best approach after PALAVRA. Inspired by our linear factorization of concepts and contexts, we propose to use a modification of AvgIm+Text where instead of averaging text and image embeddings, we add to the text embedding the \emph{difference} between mean image embeddings of the specialized concept (``my shirt'') and the mean embeddings of the general (coarse-grained) concept images (all images of shirts in the dataset). For a concrete example, if [CONCEPT] is a particular instance of a shirt, then the AvgIm+Text approach would be as follows:
\[
\begin{aligned}
&\textbf{AvgIm+Text}:\\
&\vu(\text{``A person wearing [CONCEPT] sitting on a couch})\\
&\approx \vu(\text{``A person wearing a shirt stting on a couch})\\
&\,\, + {\rm Norm}({\rm Mean}\{\vv(\text{CONCEPT})\}),
\end{aligned}
\]
where $\vu$ is the text embedding and $\vv$ is the image embedding, $\rm Mean$ means the mean over supporting samples, and $\rm Norm$ means normalization. In contrast, we propose to use the following strategy:
\[
\begin{aligned}
&\textbf{Ideal Words}:\\
&\vu(\text{``A person wearing [CONCEPT] sitting on a couch})\\
&\approx \vu(\text{``A person wearing a shirt stting on a couch})\\
&\,\, - {\rm Mean}\{\vv(\text{shirt})\} + {\rm Mean}\{\vv(\text{CONCEPT})\}.
\end{aligned}
\]
Our results are shown in~Table~\ref{tab:main_retrieval}.
Remarkably, this simple strategy that uses CLIP embeddings and \emph{does not require any training} outperforms PALAVRA by a large margin (in our experiments, we used the implementation and evaluation code provided in~\cite{cohenThisMyUnicorn2022a} with only minimal changes). This modified approach can be interpreted from the perspective of decomposable embeddings, since we are assuming that $\vu(\text{context},\text{CONCEPT}) - \vu(\text{context},\text{shirt})$ does not significantly depend on the context and can be approximated as the difference mean vectors representing the specific CONCEPT and the generic shirt. Table~\ref{tab:main_retrieval} also includes ablations for the two modifications we made \wrt to AvgIm+Text proposed in ~\cite{cohenThisMyUnicorn2022a} (\ie skipping the normalization step and removing the mean of the coarse-grained concept).

\vspace{-.3cm}

\paragraph{Visualizing ideal words.} We propose to visualize the effect of linear-algebraic operations with ideal words using a CLIP-guided diffusion model (Stable Diffusion 2.1). In this setting, we compute ideal words of decomposable strings in the same way as before (as in Proposition~\ref{prop:approximation} and Example~\ref{ex:basic-example}), with the only difference that we now consider the encoded representation of the entire string {before} the final projection layer of the text encoder (treating the concatenated token representations as a long vector), since this is required for conditioning the diffusion model.
An illustrative example is shown Figure~\ref{fig:diffusion-colors}. We mention that~\cite{wangConceptAlgebraTextControlled2023, liuCompositionalVisualGeneration2023} have also proposed algebraic manipulations to control visual generation in a compositional way; however both of those works perform operations on score functions rather than on embedding vectors, which means that their approach requires modifying the diffusion process. In contrast, similar to the prompt debiasing method from~\cite{chuangDebiasingVisionLanguageModels2023}, we simply modify the prompt embeddings that condition the generation. In this paper, we use generative models as a qualitative proof of the validity of ideal words as approximations for embeddings; we leave a detailed exploration of applying these decompositions for controlling image generation to future work.

\begin{figure}[t]
\centering
\includegraphics[width=.90\linewidth]{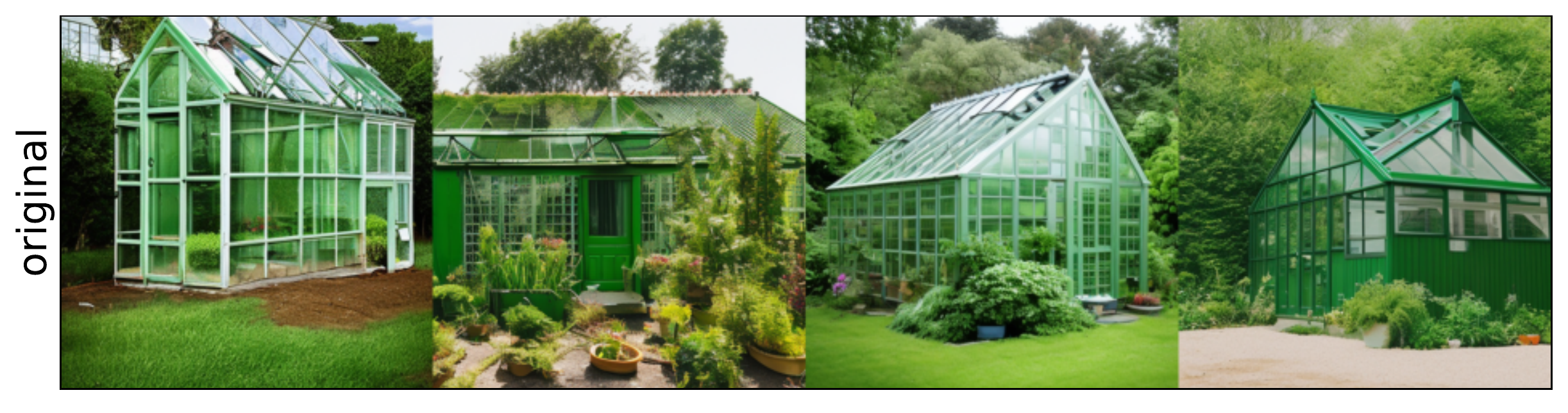}\\
\includegraphics[width=.90\linewidth]{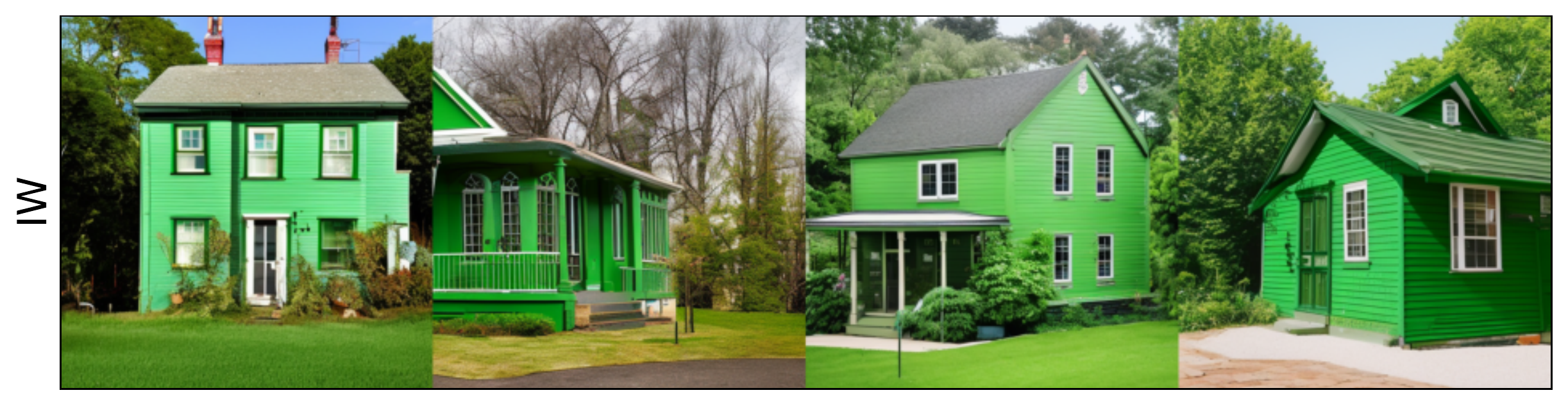}\\
\includegraphics[width=.90\linewidth]{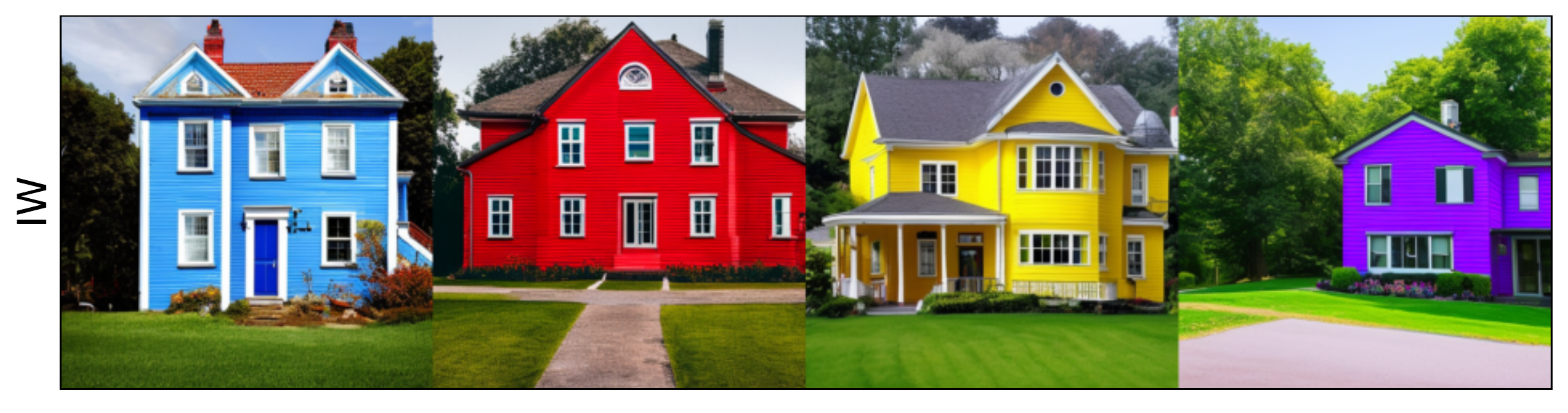}\\
\includegraphics[width=.90\linewidth]{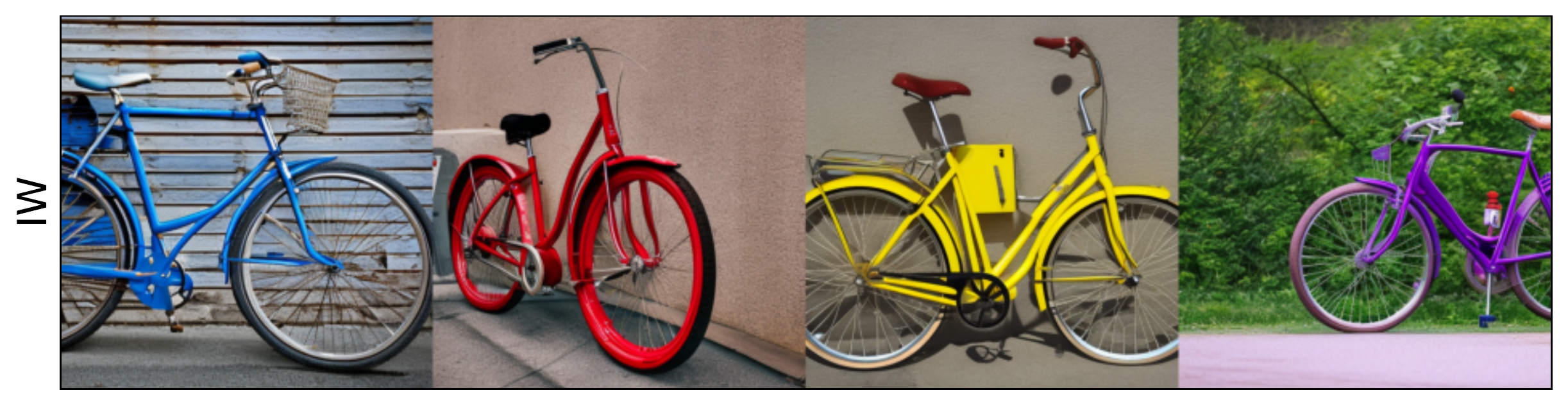}\\
\caption{\textbf{Visualization of ideal words.} \textit{First row:} images generated by Stable Diffusion with the prompt ``a photo of a green house.'' Because of the contextual encoder, ``house'' influences the meaning ``green.'' \textit{Following rows:} we compute ideal words approximations for strings of the form ``a photo of a [color] $\times$ [object],'' using five colors and four objects. In the second row, we generate images using the vector $\vu_0 + \vu_{\rm green} + \vu_{\rm house}$. Now $\vu_{\rm green}$ means green-colored because of how the string ``green'' composes with most objects. In the third row, we generate images using $\vu_0 + \vu_{\rm [color]} + \vu_{\rm house}$ for different colors; in the fourth row, we use $\vu_0 + \vu_{\rm [color]} + \vu_{\rm bike}$. The images were {not} cherry-picked or manipulated in any way. This example shows that we can generate embeddings of composite concepts by simply adding vectors in the representation space.}
\label{fig:diffusion-colors} 
\vspace{-.3cm}
\end{figure}

\section{Conclusion}

We have investigated compositional structures in VLM embeddings and argued that contextual text embeddings are often well-approximated by linear combinations of smaller sets of vectors. Optimal choices for these vectors are not embeddings of actual words, but rather ``ideal words'' that can be easily obtained as weighted averages of embeddings of longer strings of text. We showed that this simple idea can be used to design effective baseline methods for different visual language tasks (compositional classification/retrieval, debiasing, and image generation) and to control the behavior of VLMs.

In the future, we will focus on practical applications of ideal word decompositions such as compositional image generation. Furthermore, we would like to find ways of customizing ideal words using training data, for example by incorporating linear factorizations in fine-tuning strategies, or by introducing kernelized versions of these decompositions that have learnable parameters.

Finally, we remark that our discussion in Section~\ref{sec:vlm} was mainly focused on embedding vectors from a {single} modality (text), however the strategy we used for concept retrieval in Section~\ref{sec:exp} suggests that it is possible to perform linear algebraic operations using vectors from \emph{both} modalities (text/vision).
Although it is generally known that visual and text embeddings in CLIP are not well-aligned~\cite{liangMindGapUnderstanding2022}, our linear manipulations actually only require for the \emph{differences} between embedding vectors of the same modality to be aligned. 
Interestingly, this sort of weak alignment implies that vector representations of a concept $c$ in any modality can be (approximately) written as
\begin{equation}
\vw_{c} = \vw_0 \pm \vw_{\rm modality} + \ldots
\end{equation}
where $\vw_{\rm modality}$ may be seen as the ideal word vector corresponding to the modality factor for vision/text.

%%%%%%%%% REFERENCES
{\small
\bibliographystyle{ieee_fullname}
\bibliography{bibliography}
}

\clearpage
\pagenumbering{arabic}% resets `page` counter to 1
\renewcommand*{\thepage}{\arabic{page}}
\appendix

\noindent \textbf{\Large Supplementary Material}

\vspace{.6cm}

This supplementary material is organized as follows: in Section~\ref{sec:proofs} we provide proofs for all the statements of the paper and we discuss some connections with mathematical representation theory; in Section~\ref{sec:details} we give details on the datasets and prompts used for our experiments; in Section~\ref{sec:additional-results} we present some additional experimental results and qualitative examples.

\section{Proofs}
\label{sec:proofs}

\basicproperties*

\begin{proof} (1) If the vectors are decomposable, then clearly the vector differences $\vu_{z} - \vu_{z'}$ do not depend on the components that $z,z'$ share in common since the corresponding vectors cancel out. For the converse, fix $z = (z_1,\ldots,z_k) \in \mathcal Z$ arbitrarily and choose any $k$ vectors $\vu_{z_1}, \ldots, \vu_{z_k}$ such that $\vu_{z} = \vu_{z_1} + \ldots + \vu_{z_k}$. Now for any $z_i' \in \mathcal Z_i$ and any $i=1,\ldots,k$, define
\[
\begin{aligned}
&\vu_{z_i'}:= \vu_{z_i} + \vu_{z'} - \vu_{z},\\ 
&\qquad \text{where } z' = (z_1,\ldots,z_i',\ldots, z_k).
\end{aligned}
\]
If $z'' = (z_1',\ldots,z_k')$, it now holds that
\[
\begin{aligned}
&\vu_{z''} = \vu_{z''} - \vu_{(z_1,z_2',\ldots,z_k')}\\
&\qquad + (\vu_{(z_1,z_2',\ldots,z_k')} - \vu_{(z_1,z_2,\ldots,z_k')}) \\
 &\qquad + \ldots + (\vu_{(z_1,z_2,\ldots,z_k')} - \vu_{z}) + \vu_{z} \\
&\qquad= (\vu_{z_1'} - \vu_{z_1}) + \ldots + (\vu_{z_k'} - \vu_{z_k}) + \vu_{z}\\
&\qquad= \vu_{z_1'} + \ldots + \vu_{z_k'}.
\end{aligned}
\]
(2) We have that
\begin{equation}
\label{eq:remove-means}
\begin{aligned}
\sum_{z \in \mathcal Z} \gamma_z \vu_z &= \sum_{z \in \mathcal Z} \gamma_z (\vu_{z_1} + \ldots + \vu_{z_k})\\
&= \sum_{z \in \mathcal Z} \gamma_z (\bar \vu_{\mathcal Z_1} + \ldots + \bar \vu_{\mathcal Z_k} + \tilde \vu_{z_1} + \ldots \tilde \vu_{z_k}),\\
&=\sum_{z \in \mathcal Z} \gamma_z (\vu_{0} + \tilde \vu_{z_1} + \ldots \tilde \vu_{z_k}),\\
\end{aligned}
\end{equation}
where $\bar \vu_{\mathcal Z_1} := \frac{1}{n_i}\sum_{z_i \in \mathcal Z_i} \vu_{z_i}$ and $\tilde \vu_{z_i} := \vu_{z_i} - \vu_{\mathcal Z_i}$. Since $\sum_{z_i \in \mathcal Z_i} \tilde \vu_{z_i} = 0$,~\eqref{eq:remove-means} shows that any linear combination of the vectors $\vu_z, z \in \mathcal Z$ can be written as a linear combination of $1 + \sum_{i=1}^k (n_i - 1)$ vectors.
\end{proof}

\centereddecomp*

\begin{proof} Following the proof of part 2 of the previous Lemma, it is enough to let $\vu_0 := \bar \vu_{\mathcal Z_1} + \ldots + \bar \vu_{\mathcal Z_k}$ where $\bar \vu_{\mathcal Z_1} := \frac{1}{n_i}\sum_{z_i \in \mathcal Z_i} \vu_{z_i}$, and then re-center the remaining vectors accordingly. For the uniqueness, we note that~\eqref{eq:linear_decomposition} implies that the vectors~$\vu_0, \vu_{z_i}$, $z_i \in \mathcal Z_i$ satisfy
\begin{equation}
\label{eq:projection-not-weighted}
\vu_0 = \frac{1}{N} \sum_{z \in \mathcal Z} \vu_z, \quad \vu_{z_i} = \frac{n_i}{N} \sum_{\substack{z' = (z_1',\ldots, z_k')\\z_i' = z_i}} \vu_{z'} - \vu_0.
\end{equation}
where $N=n_1\ldots n_k$. In particular,~\eqref{eq:projection-not-weighted} shows that $\vu_0, \vu_{z_i}$, $z_i \in \mathcal Z_i$ are uniquely determined by the original vectors $\vu_z$.
\end{proof}

In the previous proof, we considered a map associating each $\vu_{z}, z \in \mathcal Z$ with the vectors given by
\begin{equation}\label{eq:ideal-words}
\begin{aligned}
 \vu_0 = \frac{1}{N}\sum_{z \in \mathcal Z} \vu_{z}, \,\,\, \vu_{z_i} = \frac{n_i}{N}\sum_{\substack{z' = (z_1',\ldots, z_k')\\z_i' = z_i}} \vu_{z'} - \vu_0.\\
\end{aligned}
\end{equation}
It is easy to see that if we define $\tilde \vu_z = \vu_0 + \vu_{z_1} + \ldots + \vu_{z_k}$ then applying~\eqref{eq:ideal-words} with the new vectors $\tilde \vu_{z}$ instead of $\vu_{z}$ yields the same components $\vu_{z_i}$. Thus, this map can be seen as a projection onto a decomposable set of vectors. Note that the component vectors satisfy $\sum_{z_i \in \mathcal Z_i} \vu_{z_i} = 0$. The following result considers a slightly more general setting in which these components vectors satisfy $\sum \alpha_{z_i} v_{z_i} = 0$ for some weights $\alpha_i$ that sum to $1$.

\approximation*

\begin{proof} 
Without loss of generality, we may assume that $\sum_z \beta_z \vu_{z_i} = \sum \alpha_{z_i} \vu_{z_i} = 0$. Imposing that the derivative of~\eqref{eq:best-approx} with respect to $\vu_0$ is zero leads to
\begin{equation}
\begin{aligned}
&\sum_{z \in \mathcal Z} \beta_z (\vu_z - ({\vu}_0 + \vu_{z_1} + \ldots + \vu_{z_k}))\\
&=\sum_{z \in \mathcal Z} \beta_z (\vu_z - {\vu}_0)=0,
\end{aligned}
\end{equation}
which implies $\vu_0 = {\sum_{z} \beta_z \vu_z}.$
Similarly, differentiating with respect to $\vu_{z_i}$ we have 
\begin{equation}
\begin{aligned}
&\sum_{\substack{z' = (z_1',\ldots, z_k')\\z_i' = z_i}} \beta_{z'} (\vu_{z'} - ({\vu}_0 + \vu_{z_1} + \ldots + \vu_{z_k}))\\
&=\sum_{\substack{z' = (z_1',\ldots, z_k')\\z_i' = z_i}} \beta_{z'} (\vu_{z'} - {\vu}_0 - \vu_{z_i}) = 0
\end{aligned}
\end{equation}
which implies that
\begin{equation}
\sum_{\substack{z' = (z_1',\ldots, z_k')\\z_i' = z_i}} \beta_z \vu_z  = \alpha_{z_i} (\vu_0 + \vu_{z_i}),
\end{equation}
so $\vu_{z_i}$ is as in~\eqref{eq:projection-weighted}.
\end{proof}

\groupaction*

\begin{proof} Let $r(\mathcal Z)$ be a set of decomposable vectors of maximal dimension. If $W := Span(\vu_z, z \in \mathcal Z)$, then we write $V = W \oplus W'$, and define a linear action of $G$ on $\RR^d$ by associating each group element $g = (g_1,\ldots, g_k)$ with an invertible linear transformation so that each $g_i$ determines a permutation of the vectors $\vu_{z_i}$, while fixing other terms and $W'$. This describes a disentangled action of $G$, where $V = W' \oplus \langle \vu_0 \rangle \oplus V_{\mathcal Z_1} \oplus \ldots \oplus V_{\mathcal Z_k}$ (to be consistent with the original definition, we can set $V_1 = W' \oplus \langle \vu_0 \rangle \oplus V_{\mathcal Z_1}$ and $V_i = V_{\mathcal Z_i}$ for $i \ge 2$).

For the converse, let $\rho: G \rightarrow GL(V)$ be any linear action of $G$ on $V$ (a group representation). Writing $G_{\hat i} = \mathfrak S_1 \times \ldots \times \{e\} \times  \ldots \times \mathfrak S_{k}$ (with the identity at the $i$-th component), we define
\begin{equation}
\begin{aligned}
&V_0 := \{\vu \in V \colon g \cdot \vu = \vu, \,\, \forall g \in G\},\\
&\tilde V_i := \{\vu \in V \colon g \cdot \vu = \vu, \,\, \forall g \in G_{\hat i}\}.\\
\end{aligned}
\end{equation}
Since $G$ acts linearly, these are vector spaces. We also define the linear maps
\begin{equation}\label{eq:linear_projs}
\begin{aligned}
&\pi_0: \vu \mapsto \frac{1}{|G|}\sum_{g \in G} g\cdot \vu,\\ %= \frac{1}{|\mathcal X_0|} \sum \vu_{x'}\\
&\tilde \pi_i: \vu \mapsto \frac{1}{|G_{\hat i}|}\sum_{g \in G_{\hat i}} g\cdot \vu. %= \frac{n_i}{|\mathcal X_0|}\sum_{g \in G_{\hat i}} \vu_x
\end{aligned}
\end{equation}
These are linear projections onto $V_0$ and $\tilde V_i$, respectively, since they map onto these spaces and they fix them. We now define $\pi_i:= \tilde \pi_i - \pi_0$ and $V_i := Im(\pi_i)$. Since $\tilde V_{i} \cap \tilde V_{j} = V_{0}$ for $i \ne j$, we have that $V_i \cap V_j = \{0\}$ for $i \ne j$. In general, we now have that $V_0 \oplus V_1 \oplus \ldots \oplus V_k \subset V$; if the action $\rho$ is disentangled, however, then
\begin{equation}
V = V_0 \oplus V_1 \oplus \ldots \oplus V_k.
\end{equation}
Thus, for any $v \in V$, we have $v = \pi_0(v) +  \pi_1(v) + \ldots + \pi_k(v)$. Now assume that $r: \mathcal Z \rightarrow V$ is a compositional embedding, so $g\cdot r(z) = r (g \cdot z)$. We observe that $\vu_{z_i} = \pi_i(\vu_{z})$ is fixed by $\mathfrak S_j$ for $j\ne i$, and thus depends only on $z_i$. In fact, the expressions for $\pi_0, \pi_i$ applied to $\vu_z$ are exactly the projection maps from~\eqref{eq:ideal-words}. Thus, we can write $\vu_z = \vu_{0} + \vu_{z_1} + \ldots + \vu_{z_k}$, which means that $r(\mathcal Z)$ are decomposable.
\end{proof}

\linearcomp*

\begin{proof}
Assume that~\eqref{eq:probabilistic_independence} holds, and let $g_0(y):=\log(q_0(y))$ and $g_i(z_i,y):=\log(q_i(z,y))$. For all $z \in \mathcal Z$, we can write
\begin{equation}
\begin{aligned}
\log p(x(z), y) &= g_0(y) + g_1(z_1, y) + \ldots + g_k(z_k, y)\\
&=\bar g_0(y) + \bar g_1(z_1, y) + \ldots + \bar g_k(z_k, y),\\
&s.t. \sum_{z_i \in \mathcal Z_i} \bar g_i(z_i,y) = 0, \quad i=1,\ldots,k,
\end{aligned}
\end{equation}
where $\bar g_0(y) := g_0(y) + \sum_{j=1}^k \frac{1}{n_j} \sum_{z_j \in \mathcal Z_j} g_j(z_j, y)$ and $\bar g_i(z_i, y) := g(z_i, y) - \frac{1}{n_i} \sum_{z_i' \in \mathcal Z_i} g_i(z_i', y)$. It is easy to verify the following identities for $i=1,\ldots,k$:
\begin{equation}\label{eq:g-symmetry}
\begin{aligned}
\bar g_0(y) &= \frac{1}{N} \sum_{z \in \mathcal Z} \log p(x(z), y) = \frac{1}{N} \sum_{z \in \mathcal Z} \vu_{x(z)}^\top \vv_y + c_0\\
&= \vu_0^\top \vv_y + c_0\\[.3cm]
% = g_0(y) + \sum_{j=1}^k \frac{1}{n_j} \sum_{z_j \in \mathcal Z_j} g_j(z_j, y)\\
\bar g_i(z_i, y) &= \frac{n_i}{N}  \sum_{\substack{z' = (z_1',\ldots, z_k')\\z_i' = z_i}}
 \log p(x(z), y) - \bar g_0(y)\\
 &= \frac{n_i}{N}\sum_{\substack{z' = (z_1',\ldots, z_k')\\z_i' = z_i}} \vu_{x(z')}^\top \vv_y  - \vu_{0}^\top \vv_y = \vu_{z_i}^\top \vv_y,\\
\end{aligned}
\end{equation}
where we used the expression for $\log p(x,y)$ from~\eqref{eq:vlm} and the definition of the terms $\vu_0, \vu_{z_i}$ from~\eqref{eq:ideal-words}. If we now define $\tilde \vu_{x(z)} := \vu_0 + \vu_{z_1} + \ldots + \vu_{z_k}$, then it follows from~\eqref{eq:g-symmetry} that $\tilde \vu_{x(z)}^\top \vv_y = \vu_{x(z)}^\top \vv_y (=\log p(x(z), y)) - c_0)$ for all $z \in \mathcal Z$, $y \in \mathcal Y$. Since by hypothesis $Span(\vv_y, y \in \mathcal Y) = \RR^d$, we conclude that $\tilde \vu_{x(z)} = \vu_{x(z)}$. Conversely, it is clear that if all $\vu_{x(z)}$ decompose as in~\eqref{eq:linear_decomposition}, then $p(x(z),y)$ has a factored form as in ~\eqref{eq:probabilistic_independence} for all $y \in \mathcal Y$.
\end{proof}

\linearcompcor*

\begin{proof} This follows immediately from the factored form of~\eqref{eq:probabilistic_independence}. More precisely, the statement means that
\begin{equation}\label{eq:discrminative_disentanglement}
\tilde p(z\,|\,y) = \tilde p(z_1 \,|\, y) \ldots \tilde p(z_k\,| \,y),
\end{equation}
where $\tilde p(z \,|\, y):= \frac{1}{Z_y} p(x(z) \, | \, y)$, $\tilde p(z_i \,|\, y):= \frac{1}{Z_y}\sum_{z_{k\ne i}} p(x(z) \, | \, y)$ and $Z_y:=\sum_{z} p(x(z) \, | \, y)$. We observe that~\eqref{eq:discrminative_disentanglement} implies~\eqref{eq:probabilistic_independence}, since we can write
\begin{equation}
p(x(z), y) = Z_y p(y) \tilde p(z_1|y) \ldots \tilde p(z_k|y),
\end{equation}
which has the desired factored form.
Conversely,~\eqref{eq:probabilistic_independence} means that
\begin{equation}\label{eq:constants-probability}
\tilde p(z\,|\,y) = \frac{q_0(y) Z_1 \ldots Z_k}{p(y) Z_y} \tilde q_1(z_1,y)\ldots \tilde q_k(z_k,y),
\end{equation}
where $Z_i = \sum_{z_i \in \mathcal Z_i)} q_i(z_i,y)$ and $\tilde q_i(z_i,y) = \frac{1}{Z_i} q(z_i,y)$. Since $\sum_{z \in \mathcal Z} \tilde p(z\,|\,y) = 1$, we deduce that the $y$-dependent constant on the right of~\eqref{eq:constants-probability} is equal to $1$, and $\tilde q_i(z,y) = \tilde p(z_i|y)$.
\end{proof}

\relaxation*

\begin{proof} (1) Assume that $p(x(z),y)$ is mode-disentangled. Then we have that
\begin{equation}
\begin{aligned}
&\argmax_{z_i \in \mathcal Z_i} \vu_{(z_i, z_{-i})}^\top \vv_y \\
&=\argmax_{z_i \in \mathcal Z_i} \vu_{(z_i, z_{-i}')}^\top \vv_y \\
&=\argmax_{z_i \in \mathcal Z_i} \sum_{\substack{z' = (z_1',\ldots, z_k')\\z_i' = z_i}} \vu_{z'}^\top \vv_y\\
&=\arg\max_{z_i \in \mathcal Z_i} \vu_{z_i}^\top \vv_y
\end{aligned}
\end{equation}
where $\vu_{z_i}$ is as in~\eqref{eq:ideal-words}, or as in the weighted version from~\eqref{eq:projection-weighted}. This implies that we can perform inference using the decomposable approximations $\tilde \vu_{x(z)}$ instead of the original vectors.\\
2)  We will use the notation $z = (z_i, z_j, z_{-\{i,j\}})$ where $z_{-\{i,j\}}:=(z_1,\ldots,z_{i-1}, z_{i+1}, \ldots z_{j-1}, z_{j+1}, \ldots, z_k)$.
If $p(x(z),y)$ is order-disentangled for $y$, then for any $z_i, z_i' \in \mathcal Z_i$ and $z_j, z_j' \in \mathcal Z_j$
\begin{equation}
\begin{aligned}
&(\vu_{(z_i',z_j,z_{-\{i,j\}})} - \vu_{(z_i,z_j,z_{-\{i,j\}})})^\top \vu_y \ge 0\\
&\Leftrightarrow (\vu_{(z_i',z_j',z_{-\{i,j\}})} - \vu_{(z_i,z_j',z_{-\{i,j\}})})^\top \vu_y \ge 0,\\
\end{aligned}
\end{equation}
and similarly
\begin{equation}
\begin{aligned}
&(\vu_{(z_i,z_j',z_{-\{i,j\}})} - \vu_{(z_i,z_j,z_{-\{i,j\}})})^\top \vu_y \ge 0\\
&\Leftrightarrow (\vu_{(z_i',z_j',z_{-\{i,j\}})} - \vu_{(z_i',z_j,z_{-\{i,j\}})})^\top \vu_y \ge 0.\\
\end{aligned}
\end{equation}
If these relations hold for any vector $\vu_y$, then it means that
\begin{equation}\label{eq:scalar-multiples}
\begin{aligned}
&\vu_{(z_i',z_j',z_{-\{i,j\}})} - \vu_{(z_i,z_j',z_{-\{i,j\}})}\\
&\qquad = \lambda (\vu_{(z_i',z_j,z_{-\{i,j\}})} - \vu_{(z_i,z_j,z_{-\{i,j\}})})\\
&\vu_{(z_i',z_j',z_{-\{i,j\}})} - \vu_{(z_i',z_j,z_{-\{i,j\}})}\\
&\qquad = \mu (\vu_{(z_i,z_j',z_{-\{i,j\}})} - \vu_{(z_i,z_j,z_{-\{i,j\}})})\\
\end{aligned}
\end{equation}
for some positive scalars $\lambda, \mu \in \RR$. It follows from Lemma~\ref{lem:aligned-differences} below that either all four points in~\eqref{eq:scalar-multiples} are aligned, or $\lambda = \mu = 1$. However, we can exclude that all four points are aligned for otherwise the largest between $p(x(z_i,z_j,z_{-\{i,j\}}), y)$ and $p(x(z_i',z_j,z_{-\{i,j\}}), y)$ would determine the largest among $p(x(z_i,z_j,z_{-\{i,j\}}), y)$ and $p(x(z_i,z_j',z_{-\{i,j\}}), y)$, \ie, the factors $\mathcal Z_i, \mathcal Z_j$ would not be distinct. (Technically, we can assume in our definition of ``factors'' that all possible rankings of values of $\mathcal Z_i$ are possible for any choice of $z_{-i}$). Thus, $\lambda = \mu = 1$ in ~\eqref{eq:scalar-multiples} for all $z_i, z_i', z_j, z_j'$. This implies that $\vu_{(z_i, z_{-i})} - \vu_{(z_i', z_{-i})}$ does not depend on $z_{-i}$, which in turn means that the vectors $\vu_{z}$ are decomposable, since $\vu_{z} - \vu_{z'}$ does not depend on components that $z, z'$ have in common.
\end{proof}

\begin{lemma}\label{lem:aligned-differences} If $\vp, \vq, \vr, \vs \in \RR^d$ are such that
\begin{equation}\label{eq:proportial-differences}
\begin{aligned}
&\vp - \vq = \lambda (\vr - \vs),\quad \vp - \vr = \mu (\vq - \vs),\\
\end{aligned}
\end{equation}
for some scalars $\lambda, \mu \in \RR$, then either $\vp, \vq, \vr, \vs$ lie on the same affine line (\ie, all pairwise differences are scalar multiples of each other) or $\lambda=\mu=1$.
\end{lemma}
\begin{proof} Substituting $\vp = \vq + \lambda(\vr - \vs)$ in the second equality in~\eqref{eq:proportial-differences} yields
\begin{equation}
(1-\mu)\vq + (\lambda-1) \vr + (\mu - \lambda)\vs = 0.
\end{equation}
If $\mu \ne 1$ or $\nu \ne 1$, then this shows that $\vp, \vr, \vs$ are aligned (note that coefficients sum to $1$). Using the relation for $\vp$, we conclude that either $\mu=\nu=1$ or all four points are aligned.
\end{proof}

We conclude this section by elaborating on the connection with mathematical representation theory. 
This discussion is not necessary for understanding the paper, but we believe that the symmetry-based viewpoint introduced in~\cite{higginsDefinitionDisentangledRepresentations2018a} is a useful framework for studying disentanglement and compositionality in machine learning. For convenience to the reader, we include here a minimal set of definitions and basic results from representation theory, focusing on the representation of finite groups. More details can be found, for example, in~\cite{fultonRepresentationTheory2004}.

A \emph{representation} of a group $G$ is a homomorphism $\rho: G \rightarrow GL(V)$, where $V$ is a finite-dimensional vector space (typically over the complex numbers, but we can focus on the the real setting here). Often the map $\rho$ is omitted and the representation is identified with $V$. It also common to say that $V$ is a ``$G$-module'' or a ``$G$-representation.'' Given two $G$-representations $V, W$, a \emph{homomorphism of representations} is a linear map $\varphi: V \rightarrow W$ that is $G$-equivariant:
\begin{equation}
\varphi(g \cdot v) = g \cdot \varphi(v), \quad \forall g \in G,\, v \in V.
\end{equation}
A \emph{subrepresentation} (or \emph{submodule}) of a $G$-representation $V$ is a vector subspace $H \subset V$ such that is $G$-invariant:
\begin{equation}
g(h) \in H, \qquad \forall g \in G, \,h \in H.
\end{equation}
If $\varphi: V \rightarrow W$ is a homorphism of representations, then the kernel and image of $\varphi$ are subrepresentations of $V$ and $W$, respectively. A $G$-representation of $V$ is \emph{irreducible} if it has no proper subrepresentations, \ie, if its only subrepresentations are $\{0\}$ and itself.

\begin{example}[Trivial representation] Let $G$ be any group and let $V = \RR$ be a one-dimensional vector space. Then the map $\rho: G \rightarrow GL(V)$ that every element of $G$ with the identity on $V$ is an irreducible representation, called the \emph{trivial representation}.
\end{example}

\begin{example}[Permutation representation]\label{ex:permutation} Let $V = \RR^n$ and consider the representation $\rho: \mathfrak S_n \rightarrow GL(V)$ that permutes coordinates. This is not an irreducible representation since the one-dimensional subspace $V_0 = \langle (1,\ldots,1) \rangle$ is a subrepresentation (a ``copy'' of the trivial representation). In fact, we have that $V = V_0 \oplus V_1$ where $V_{1} = \{v \colon v_1 + \ldots + v_n = 0\}$. One can show that $V_1$ is irreducible, and it is called the \emph{standard representation} of $\mathfrak S_n$.
\end{example}

The next statements imply that, for finite groups, irreducible representations can always be used as ``building blocks'' for describing arbitrary representations. The irreducible components of a representation are (nearly) uniquely determined; moreover, there are only finitely many irreducible representations of a group up to isomorphism.

\begin{proposition}[Corollary 1.6, \cite{fultonRepresentationTheory2004}] If $G$ is a finite group, any $G$-representation can be decomposed as a direct sum of irreducible representations. 
\end{proposition}

\begin{proposition}[Proposition 1.8, \cite{fultonRepresentationTheory2004}] Let $V$ be a $G$-representation, and consider its decomposition into irreducible representations:
\begin{equation}
    V = V_1^{\oplus a_1} \oplus \ldots \oplus V_k^{\oplus a_k}.
\end{equation}
Then the spaces $V_i^{\oplus a_i}$ are uniquely determined. The irreducible representations $V_i$ are determined up to isomoprhism.
\end{proposition}

\begin{proposition}[Corollary 2.18, \cite{fultonRepresentationTheory2004}] Every finite group only has a finite set of irreducible representations, up to isomorphism.
\end{proposition}

For example, the irreducible representations of a symmetric group $\mathfrak S_n$ are in one-to-one correspondence with the (unordered) partitions of $n$ elements. See~\cite[Chapter 4]{fultonRepresentationTheory2004} for an explicit description.

We now return to our factored set $\mathcal Z = \mathcal Z_1 \times \ldots \times \mathcal Z_k$. We consider the vector space $\langle \mathcal Z \rangle = Span(\ve_z \colon z \in \mathcal Z)$, spanned by independent basis vectors associated with elements of $\mathcal Z$. We can identify $\langle \mathcal Z \rangle$ with the space $\RR^{n_1} \otimes \ldots \otimes \RR^{n_k}$. 
As $\mathfrak S_i$-modules, $\RR^{n_i} \cong V_{0, n_i} \oplus V_{1, n_i}$ where $V_{0, n_i}$ is a trivial representation and $V_{1, n_i}$ is the standard representation for $\mathfrak S_i$. We thus have that
\begin{equation}\label{eq:tensor_expansion}
\begin{aligned}
\langle \mathcal Z \rangle &\cong \bigotimes_{i=1}^k \left(V_{0,n_i} \oplus V_{1, n_i}\right)\\
&\cong \bigoplus_{\epsilon_i \in \{0,1\}} V_{\epsilon_1,n_1} \otimes  \ldots \otimes V_{\epsilon_k, n_k}\\
&\cong \bigoplus_{\epsilon \in \{0,1\}^k} V_{\epsilon}, 
\end{aligned}
\end{equation}
with $V_\epsilon := V_{\epsilon_1,n_1} \otimes  \ldots \otimes V_{\epsilon_k, n_k}$. This is a decomposition of $\langle \mathcal Z \rangle$ into irreducible $G$-representations (see \cite[Exercise 2.36]{fultonRepresentationTheory2004}). We can describe the projection $\pi_\epsilon$ onto $V_\epsilon$ explicitly
\begin{equation}\label{eq:proj}
\pi_{\epsilon} = \pi_{\epsilon_1, n_1} \otimes \ldots \otimes \pi_{\epsilon_k, n_k},
\end{equation}
where $\pi_{\epsilon, n_i}: \RR^{n_i} \rightarrow \RR^{n_i}$ are given by
\begin{equation}
\begin{aligned}
&\pi_{0, n_i}(\vu) := \frac{1}{|\mathfrak S_i|} \sum_{g \in \mathfrak S_i} g \cdot \vu,\\
&\pi_{1, n_i}(\vu) := \vu - \pi_{0, n_i}(\vu).
\end{aligned}
\end{equation}

A data embedding $r: \mathcal Z \rightarrow \RR^d$ can be uniquely associated with a linear map $\langle r \rangle : \langle \mathcal Z \rangle \rightarrow \RR^d$ or can equivalently be viewed as a tensor in $[r] \in \RR^{n_1} \otimes \ldots \otimes \RR^{n_k} \otimes \RR^d$.
% \footnote{More accurately, this should be $\RR^{n_1 \ast} \otimes \ldots \otimes \RR^{n_k \ast} \otimes \RR^d$, but we identify vector spaces and their duals for simplicity.} 
The image of $\langle r \rangle$ is a $G$-module in $\RR^d$ and its decomposition will contain a subset of the irreducible components in~\eqref{eq:tensor_expansion}. The notion of disentangled representation given in~\cite{higginsDefinitionDisentangledRepresentations2018a} means that the only irreducible components that contribute to the image of $r$ are the representations $V_{\epsilon}$ such that $\epsilon_i=1$ for at most one index $i$. Equivalently, we require that the projection of the image of $r$ onto the ``entangled components'' is zero, \ie, $\pi_\epsilon(\vu_z)=0$ whenever $|\{i \colon \epsilon_i=1\}|>1$. 
An intuitive way to understand this notion is in terms of the tensor $[r] \in \RR^{n_1} \otimes \ldots \otimes \RR^{n_k} \otimes \RR^d$: we require that each of the $d$ ``slices'' $\RR^{n_1} \otimes \ldots \otimes \RR^{n_k}$ can be obtained by summing ``one-dimensional slices'' of the form $\bm{1} \otimes \ldots \otimes \vu_i \otimes \ldots \otimes \bm{1}$ (similar to summing vectors into a tensor by ``array broadcasting''). In fact, this observation leads to the following characterization of linear factorization in terms of tensor-rank.

\begin{proposition} A tensor $[r] \in \RR^{n_1} \otimes \ldots \otimes \RR^{n_k} \otimes \RR^d$ corresponds to a decomposable representation if and only if all $(\RR^{n_1} \otimes \ldots \otimes \RR^{n_k})$-slices of $\exp([r])$ have tensor-rank one, where $\exp([r])$ is obtained from $[r]$ by exponentiating element-wise. This is true if and only if for all $\varphi \in (\RR^{d})\ast$ $\exp (\varphi([r]))$ has tensor-rank one.
\end{proposition}
\begin{proof}[Proof sketch.] The first claim follows from the previous discussion and the fact that $\exp\left(\sum_i \bm{1} \otimes \ldots \otimes \vu_i \otimes \ldots \otimes \bm{1}\right) = \exp(\vu_1) \otimes \ldots \otimes \exp(\vu_k)$. For the second statement, we note $\exp(t)$ having rank-one is a linear condition on a tensor~$t$.
\end{proof}

For categorical distributions of multiple variables, the distribution tensor having rank equal to one corresponds to statistical independence of variables, so the result above can be seen as an algebraic reformulation of Proposition~\ref{prop:linear_composition} in the main body of the paper. 
We also note that that other probabilistic conditions could be considered by allowing for more irreducible components in from~\eqref{eq:tensor_expansion} to appear in the image of $r$. This is similar to the log-linear representations of multivariate data. In fact, it is possible to express any conditional independence assumption on $p(\mathcal Z|\mathcal Y)$ in terms of linear-algebraic conditions on the data representation $r$.

\section{Experimental Details}
\label{sec:details}

\paragraph{Datasets.} The MIT-states dataset~\cite{isolaDiscoveringStatesTransformations2015} contains images of 245 objects modified by 115 adjectives, for a total of 28175 classes. The test set has size 12995. The UTZappos dataset~\cite{yuFineGrainedVisualComparisons2014} contains images of 12 shoe types with 16 fine-grained states. The test set has size 2914. Note that in both of these datasets only a small portion of all possible attribute-object pairs actually occurs in the test set. However, in our experiments we assume that we do not have access to this information. We also mention that prior works that have used these datasets such as~\cite{manciniLearningGraphEmbeddings2022, nayakLearningComposeSoft2022} have differentiatied between the performance on label pairs that were seen in training and those that were not. Since this distinction is not relevant in a zero-shot setting, we simply report accuracy on objects, attributes, and attribute-object pairs. In the Waterbird dataset~\cite{sagawaDistributionallyRobustNeural2020a} labels are ``waterbird/landbird'' and spurious attributes are ``water background/land background.'' There are 5794 test samples divided in four unbalanced groups. On the CelebA~\cite{liuDeepLearningFace2015}, labels are ``not blond/blond'' and spurious attributes are ``male/female.'' There are a total 19962 test samples with unbalanced groups. 
% Note that in the case of Waterbirds, one may question whether
The DeepFashion2 dataset~\cite{DeepFashion2} with the captions provided in PerVL~\cite{cohenThisMyUnicorn2022a} contains 1700 images from 100 unique fashion items. Following~\cite{cohenThisMyUnicorn2022a} val/test splitting, we retrieve 50 of these concepts selected for testing. We use 5 randomly chosen images per fashion item as per-concept supporting images, and use a test set with 221 images containing all 50 concepts and their captions (see~\cite{cohenThisMyUnicorn2022a} for more details). Final results are obtained by averaging the Mean Reciprocal Rank metric over 5 random seeds.

\paragraph{Prompts.} For MIT-States and UTZappos, we use the prompt ``image of a [$a$][$o$],'' ``image of a [$a$] object,'' and ``image of a [$o$],'' as explained in the main body of the paper. Here [$a$] and [$o$] are the lower-case original class labels.\footnote{In the case of objects for UTZappos, we perform a simple split `Boots.Mid-Calf' $\rightarrow$ ``boots mid-calf''}
For our experiments on debiasing on the Waterbirds and CelebA datasets we use the same prompts and spurious attributes used in~\cite{chuangDebiasingVisionLanguageModels2023}. These are shown in Tables~\ref{tab:prompt_waterbird} and~\ref{tab:prompt_celebA}. To compute debiased prompts we simply prepend all spurious prompts to each class prompts and then average the spurious prompts to obtain debiased class prompts (note that spurious prompts are ``balanced'' in their bias); this simpler but conceptually similar to the ``Orth-Proj'' approach used in in~\cite{chuangDebiasingVisionLanguageModels2023} that computes an orthogonal projection in the orthogonal complement of the linear space spanned by the spurious prompts. We do not make use of the ``positive pairs'' of prompts that are used in that work for regularization of the projection map.

\begin{table}[t]
\small
\centering
\begin{tabular}{c}
\toprule
 \textbf{Class Prompts}
  \\
  \midrule
 This is a picture of a landbird.
\\
 This is a picture of a waterbird.
\\
  \midrule
\textbf{Spurious Prompts}
\\
  \midrule
 This is a land background. \; This is a picture of a forest.  \\
 This is a picture of a moutain. \; This is a picture of a wood.
 \\
 This is a water background. \; This is a picture of an ocean. \\ 
This is a picture of a beach. \; This is a picture of a port.
 \\
 \bottomrule
\end{tabular}
\caption{{Prompts for Waterbird dataset~\cite{sagawaDistributionallyRobustNeural2020a} from~\cite{chuangDebiasingVisionLanguageModels2023}.}
}
\vspace{-0mm}
\label{tab:prompt_waterbird}
\end{table}

\begin{table}[t]
\centering
\small
\begin{tabular}{c}
\toprule
 \textbf{Class Prompts}
  \\
  \midrule
 A photo of a celebrity with dark hair.
\\
 A photo of a celebrity with blond hair.
\\
  \midrule
\textbf{Spurious Prompts}
\\
  \midrule
A photo of a male. \; A photo of a male celebrity.\\
A photo of a man. \; A photo of a female. \\
A photo of a female celebrity. \; A photo of a woman. 
 \\
\bottomrule
\end{tabular}
\caption{Prompts for CelebA dataset~\cite{liuDeepLearningFace2015} from~\cite{chuangDebiasingVisionLanguageModels2023}. 
}
\label{tab:prompt_celebA}
\vspace{3mm}
\end{table}

\section{Additional Results and Discussions}
\label{sec:additional-results}

\paragraph{Quantifying compositionality.} Given a set of vectors $\vu_z, z \in \mathcal Z$ in $\RR^d$, we can measure how close the vectors are to being decomposable by using
\begin{equation}\label{eq:sq-dist}
\begin{aligned}
&D(\vu_z, z \in \mathcal Z) := \min_{\tilde \vu_z} \frac{1}{|\mathcal Z|}\sum_{z \in \mathcal Z} \|\vu_{z} - \tilde \vu_{z}\|^2,\\
&\quad s.t. \,\,\{\tilde \vu_z\} \text{ is decomposable}.
\end{aligned}
\end{equation}
The optimal vectors $\tilde \vu_z$ here are the ideal word approximations given by Proposition~\ref{prop:approximation}.
In Table~\ref{tab:euclid-distance}, we report this quantity for embeddings of objects-attributes in the datasets~MIT-States~\cite{isolaDiscoveringStatesTransformations2015} and UT Zappos~\cite{yuFineGrainedVisualComparisons2014} (IW column). For comparison, we also include the average squared distance between the original embeddings and the average of the individual object and attribute embddings based on ``real words'' (RW column), and the average squared distance between pairs of the original embedding vectors (Avg). In Table~\ref{tab:euclid-distance-random}, we report the same quantities but using embeddings obtained from a \emph{randomly initialized} encoder. These results suggest that embeddings at initialization are already compositional. We discuss this point further in the next paragraph. % and in Section~\ref{sec:discussions}.

\begin{table}[t]
\small
\centering
\begin{tabular}{@{}lccc@{}}
\toprule
&\textbf{IW} & \textbf{RW} & \textbf{Avg} \\ \midrule
MIT-States~\cite{isolaDiscoveringStatesTransformations2015} &    0.23 $\pm$ 0.05     & 0.43 $\pm$ 0.06  & 0.78 $\pm$ 0.13    \\
UT Zappos~\cite{yuFineGrainedVisualComparisons2014}   &   0.16 $\pm$ 0.04     & 0.51 $\pm$ 0.05  & 0.58 $\pm$ 0.18    \\
\bottomrule
\end{tabular}
\caption{Quantifying compositionality using a trained encoder.}
\label{tab:euclid-distance}
\space{.5cm}
\end{table}

\begin{table}[t]
\small
\centering
\begin{tabular}{@{}lccc@{}}
\toprule
&\textbf{IW} & \textbf{RW} & \textbf{Avg} \\ \midrule
MIT-States~\cite{isolaDiscoveringStatesTransformations2015} &    0.04 $\pm$ 0.02     & 0.16 $\pm$ 0.02  & 0.10 $\pm$ 0.03    \\
UT Zappos~\cite{yuFineGrainedVisualComparisons2014}   &   0.10 $\pm$ 0.02     & 0.22 $\pm$ 0.04  & 0.14 $\pm$ 0.05    \\
\bottomrule
\end{tabular}
\caption{Quantifying compositionality using a randomly initialized encoder.}
\label{tab:euclid-distance-random}
\end{table}

\paragraph{Visualized embeddings.} We present more examples of projected embeddings of composite strings. In Figure~\ref{fig:3dplots-sup-mat}, we consider again the four manually constructed examples from Figure~\ref{fig:3dplots} in the main body of the paper: ``a photo of a \{red, blue, pink\} $\times$ \{car, house\}''; ``a photo of a \{big, small\} $\times$ \{cat, dog\} $\times$ \{eating, drinking\}''; ``\{a photo of a, a picture of a\} $\times$ \{place, object, person\}''; ``king, queen, man, woman, boy, girl.'' The top row of Figure~\ref{fig:3dplots-sup-mat} is the same as the top row from Figure~\ref{fig:3dplots}. In the bottom row of~\ref{fig:3dplots-sup-mat}, we visualize the embeddings of the same strings using a randomly initialized text encoder. In the first three examples, the factored structure is also \emph{syntactic}, \ie, it is based on the string structure. In these cases, the embeddings remain roughly decomposable even with random encoder. In the last case, however, decomposable structures are not visible anymore, since the strings in this example contain no repeated substrings. Note also that in third case, the factor corresponding to \{a photo of a, a picture of a\} is no longer ``squashed'' since these two strings not considered similar by the randomly initialized encoder. 

We show other examples of this effect in Figure~\ref{fig:3d-rand}. Here each pair of plots shows projections of the same strings using a trained encoder (left figure) and a randomly initialized encoder (right figure). As one might expect, for strings corresponding to capital-country relation (first row), the approximate symmetries that can be seen in the embbedings from the trained encoder are no longer present when using the random encoder. The strings in the second row, however, have a synctatic factored structure. In this case, we visually observe strong symmetries in the embeddings from the trained encoder as well as from the random encoder.

In Figure~\ref{fig:compositional-strings}, we consider 2D projections of embeddings of factored strings that include idioms such as ``cold shoulder,'' ``big apple'', ``black friday,'' ``hot pepper.'' We compare these embeddings with those of similar factored strings in which meanings of words are more conventional and uniform. In both cases, we quantify the amount of linear compositionality both visually and using the squared residual as in~\eqref{eq:sq-dist}. The results confirm the natural intuition that linear compositionality is measurably weaker when strong contextual effects between words are present.

\begin{figure}
    \centering
    \includegraphics[width=.98\columnwidth]{3d_plots4.pdf}\\[.25cm]
    \includegraphics[width=.98\columnwidth]{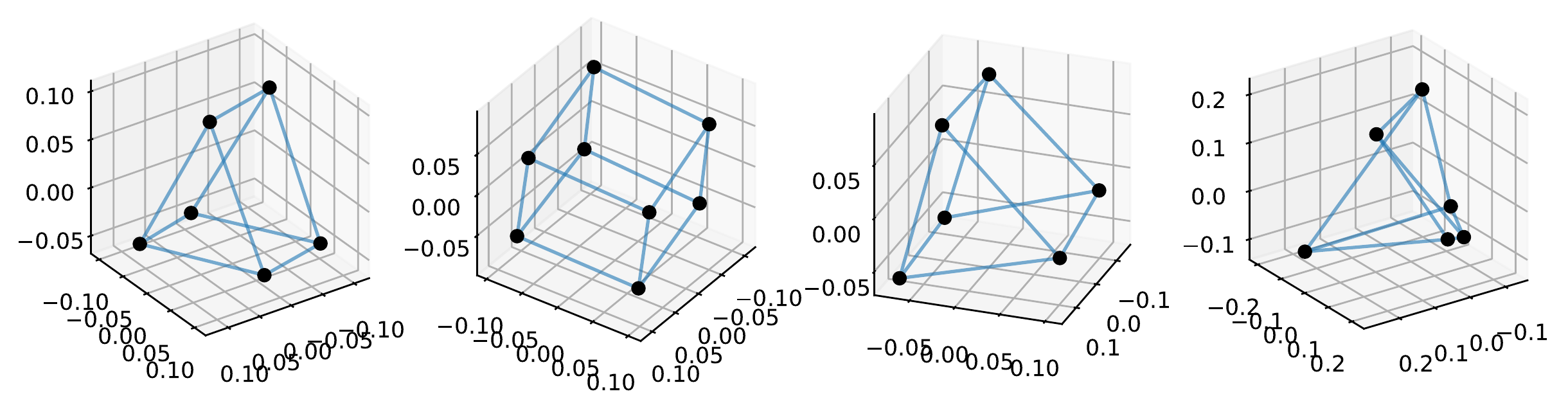}
    \caption{Projected embeddings of manually constructed strings associated with factored concepts, as described in Section~\ref{sec:exp} in the main body of the paper. \textit{Top:} trained encoder (same as in Figure~\ref{fig:3dplots}). \textit{Bottom:} visualization of the embeddings for the same strings using a randomly initialized encoder. Even without semantic information, the embeddings in the first three examples are still roughly decomposable.}
    \label{fig:3dplots-sup-mat}
\end{figure}

\begin{figure}[t]
    \centering
    \fbox{\includegraphics[width=.46\columnwidth]{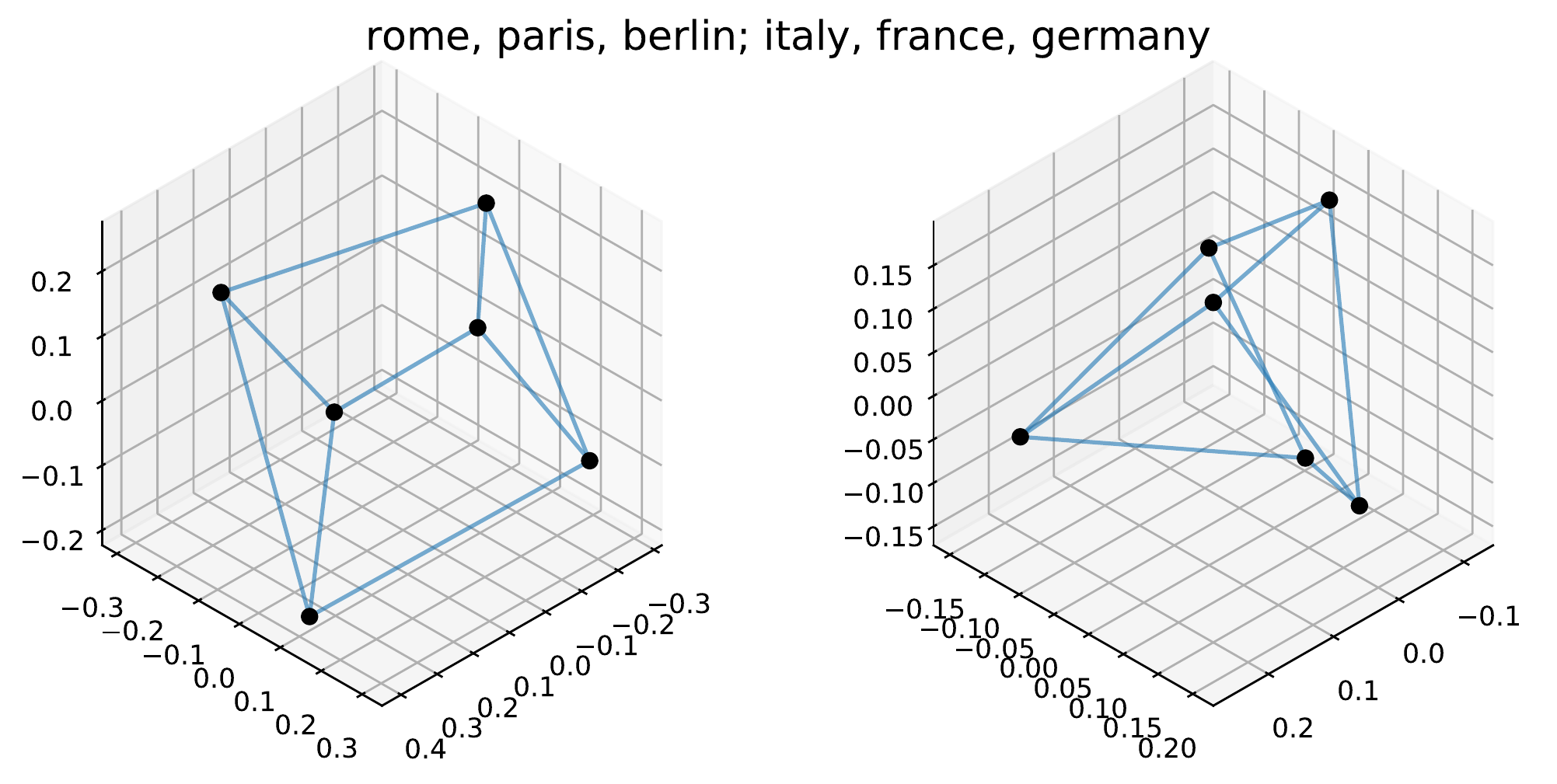}}
    \fbox{\includegraphics[width=.46\columnwidth]{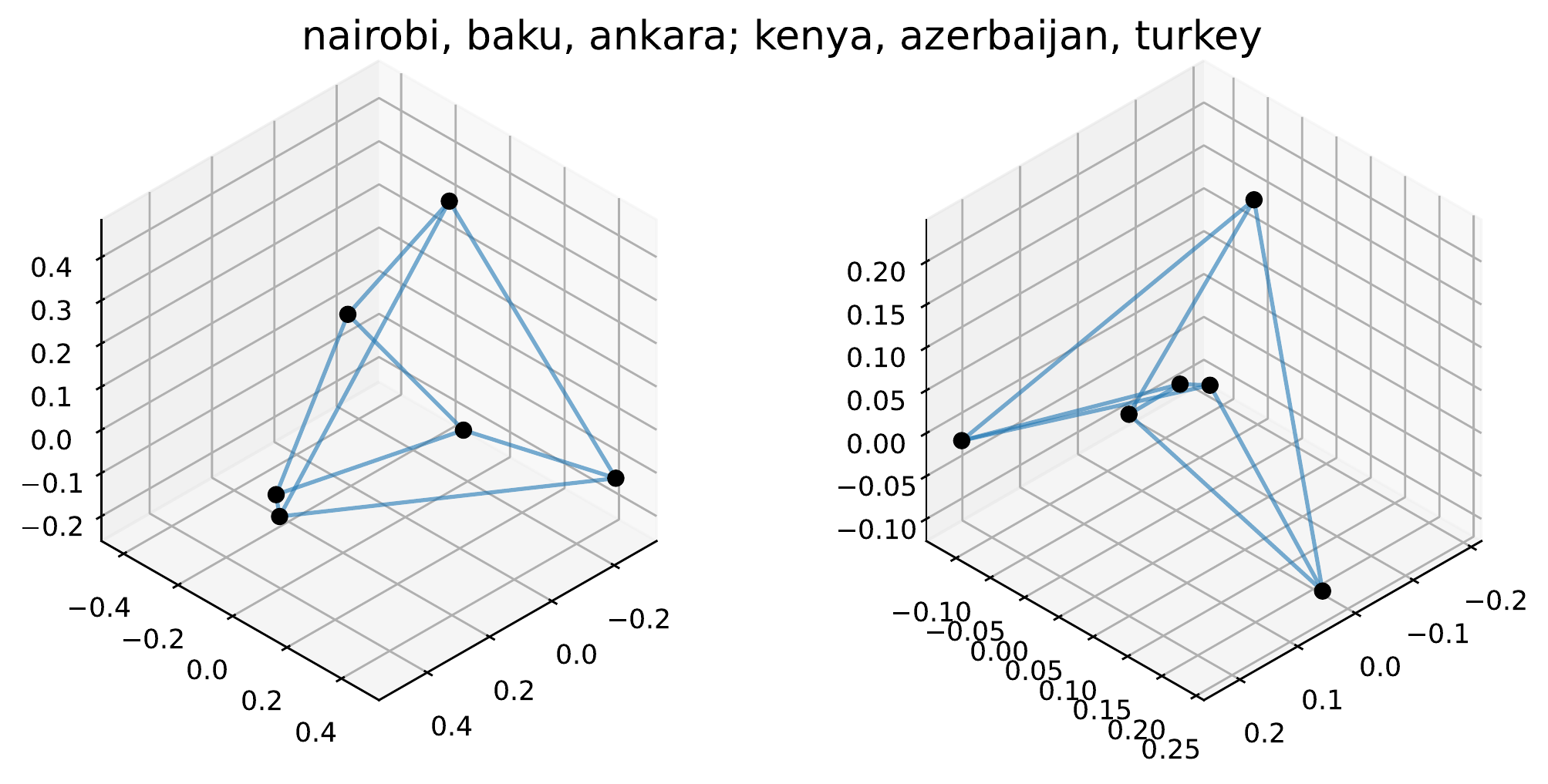}}\\[.2cm]
    % \fbox{\includegraphics[width=.46\columnwidth]{figures/2_plots_rand.pdf}} 
    % \fbox{\includegraphics[width=.46\columnwidth]{figures/3_plots_rand.pdf}}\\[.2cm] 
    \fbox{\includegraphics[width=.46\columnwidth]{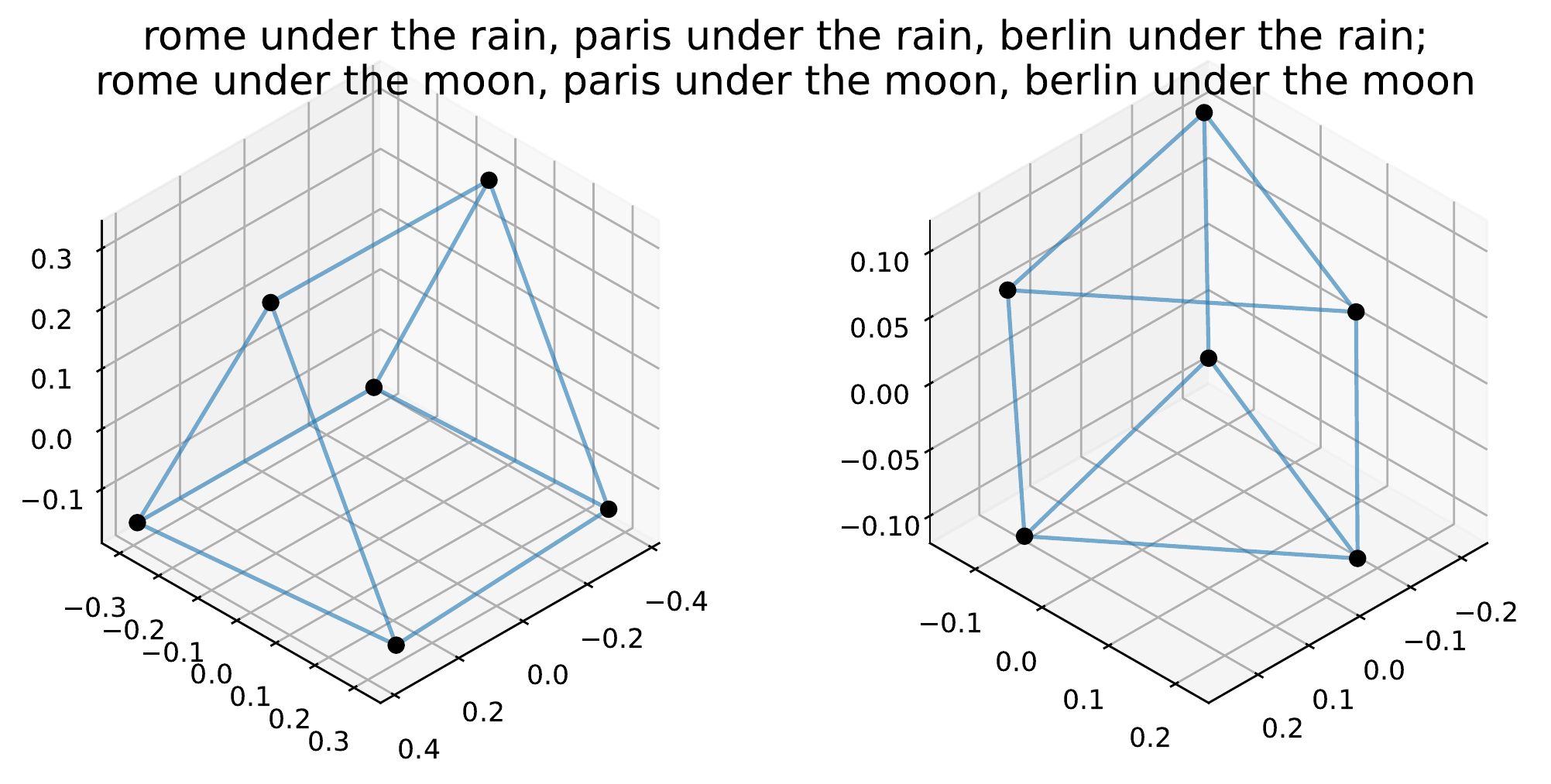}}
    \fbox{\includegraphics[width=.46\columnwidth]{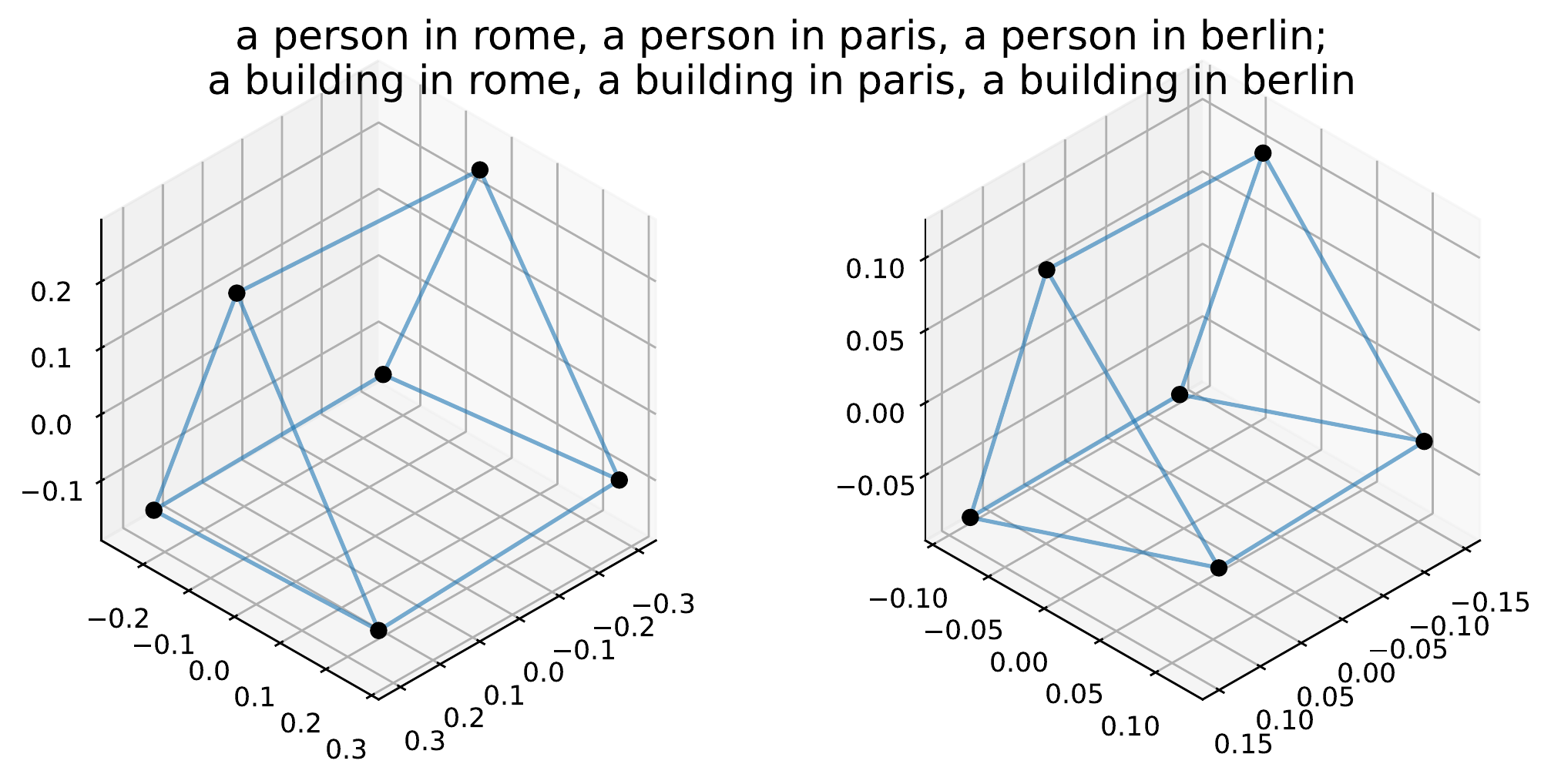}}
    \caption{Comparison between projected embeddings using a trained encoder (left figure in each pair) and using a randomly encoder (right figure in each pair). Both encoders lead to symmetric structures when the strings have a factored syntax (bottom row), while only the trained encoder shows these approximate structures when the factorization is semantic (top row).}
    \label{fig:3d-rand}
\end{figure}

\begin{figure}[t]
    \centering
    \fbox{\includegraphics[width=.46\columnwidth]{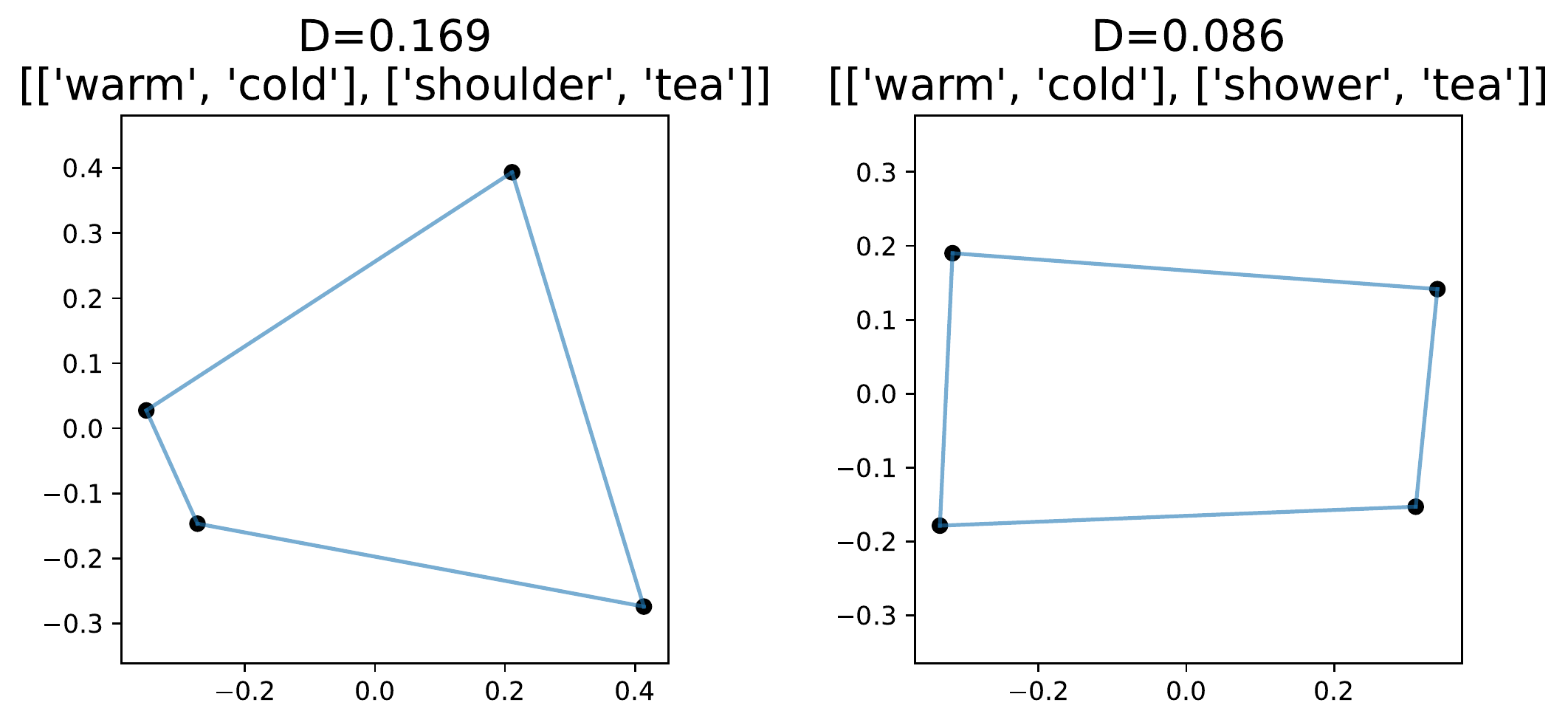}}\,
    \fbox{\includegraphics[width=.46\columnwidth]{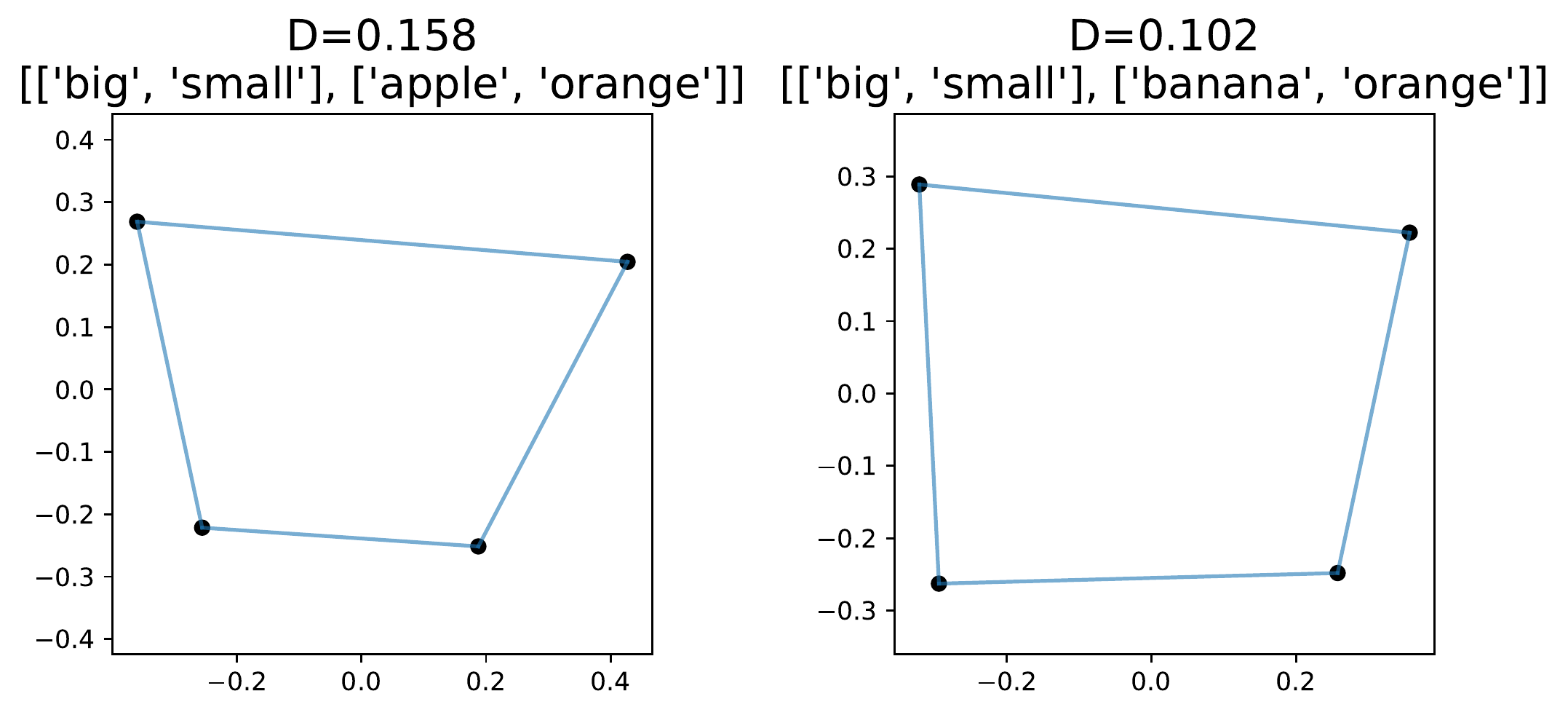}}\\[.2cm]
    \fbox{\includegraphics[width=.46\columnwidth]{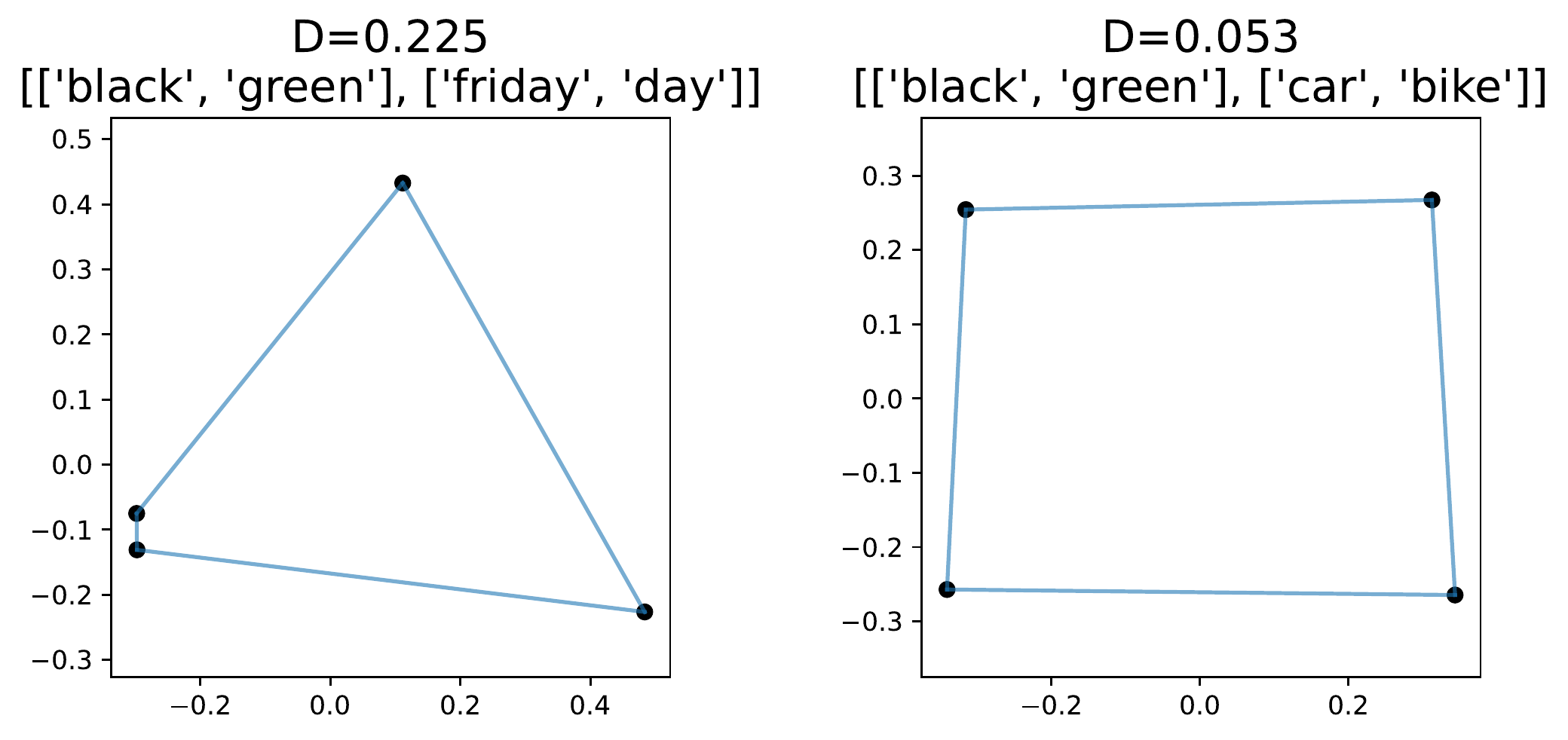}}\,
    \fbox{\includegraphics[width=.46\columnwidth]{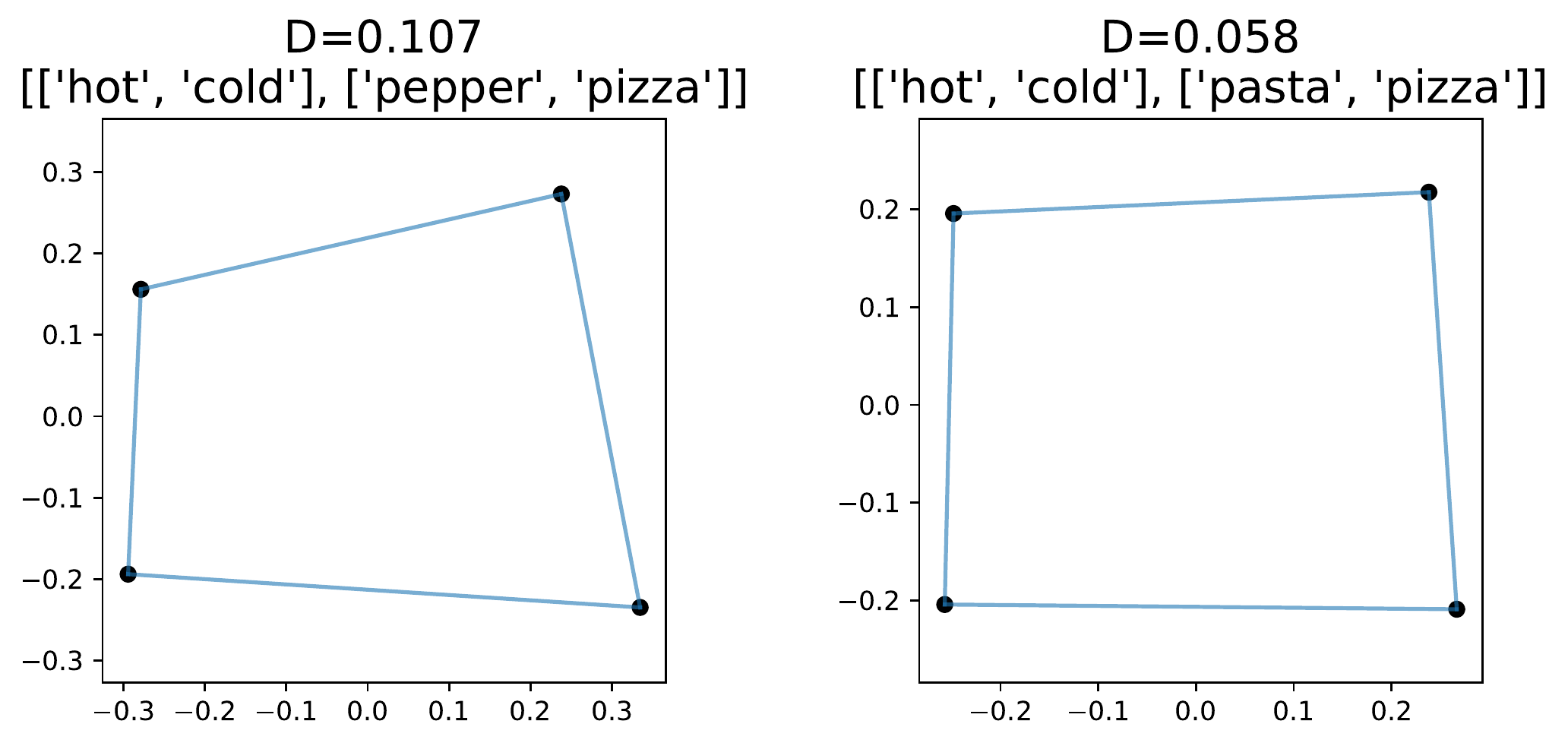}}\\[.2cm]
    \caption{Comparison between projected embeddings for factored strings with and without idioms that have non-compositional meaning (left and right in the subfigures, respectively). We can qualitatively and quantitatively see that idioms lead to weaker compositionality.}
    \label{fig:compositional-strings}
\end{figure}

\paragraph{Other notions of probabilistic disentanglement.} Proposition~\ref{prop:linear_composition} shows that linear factorization of embeddings corresponds to conditional independence of factors $z_i$ given the image $y$. One might also consider a different sort of probabilistic disentanglement in which conditionals are reversed:
\begin{equation}\label{eq:causal-disentanglement}
p(y|z = (z_1,\ldots,z_k)) = p(y|z_1) \ldots p(y|z_k) q_0(y).
% p(y|x( z_1,\ldots,z_k)) = q_0(y) q_1(z_1,y) \ldots q_k(z_k,y).
\end{equation}
This can be viewed as a sort of ``causal disentanglement'' (similar to the notion used in~\cite{wangConceptAlgebraTextControlled2023}). It follows from Corollary~\ref{cor:probability-disentanglement} that decomposable embeddings mean that
\begin{equation}\label{eq:causal-disentanglement-factor}
p(y|z) = p(y|z_1)\ldots p(y|z_k) p(y)^{1-k} \frac{p({z_1})\ldots p({z_k})}{p(z_1,\ldots,z_k)}.
\end{equation}
Thus, conditional independence has the same form as~\eqref{eq:causal-disentanglement} up to the factor $\frac{p({z_1})\ldots p({z_k})}{p(z_1,\ldots,z_k)}$ (pointwise mutual information) that does not depend on $y$. If factors are globally independent, then~\eqref{eq:causal-disentanglement-factor} and~\eqref{eq:causal-disentanglement} are equivalent. It is also worth noting that~\eqref{eq:causal-disentanglement} does not determine the marginal distribution $p(z=(z_1,\ldots,z_k))$. % so we may assume that $z_i$ are sampled independently or uniformly. 
In general, linear factorization of the embeddings can be seen as a relaxed version of causal disentanglement.

\paragraph{Normalization.} Embedding vectors for CLIP are typically normalized, however ideal word vectors are \emph{never} normalized. While this may appear strange, we note that the norm of the embeddings does not carry a probabilistic meaning: we can replace the embeddings $\vu, \vv$ from the two modalities with ${\bm T} \vu$ and ${\bm T}^{-1} \vv$ for any invertible linear transformation $\bm T$ of $\RR^d$ without changing the probability model on $\mathcal X \times \mathcal Y$. In general, ideal word manipulations require starting from normalized embeddings for consistency between modalities, but then normalization is never applied again (in fact, the inner product structure on the embedding space is not used). This explains our modification to the AvgIm+Text approach in Section~\ref{sec:exp} in the paper.

\paragraph{Visualizations using SD.} We present a few additional visualizations of ideal words using Stable Diffusion. In Figure~\ref{fig:green-amplify}, we consider the same ideal word approximation as in Figure~\ref{fig:diffusion-colors} in the main body of the paper and observe the effect of scaling the ideal word corresponding to ``green.'' That is, we consider $\vu_0 + \vu_{\rm house} + \gamma \cdot \vu_{\rm green}$ for different $\gamma$. In the top row, we compute $\vu_{\rm green}$ using the standard ``balanced'' computation for ideal words (uniform $\alpha_i$ in Proposition~\ref{prop:approximation}). In the bottom row, we use weights $\alpha_{\rm house} = 1$ and $\alpha_{\rm obj} = 0$ otherwise. This implies that the IW corresponding to $\vu_{\rm green}$ is determined by how ``green'' composes with ``house.'' Amplifying $\vu_{\rm green}$ now increases the ``greenhouse-ness'' of the generated image.

In Figure~\ref{fig:iw-transfer}, we consider the problem of \emph{transferring} ideal words. That is, we consider a different (\ie, totally disjoint) set of objects and colors compared to the ones used for Figure~\ref{fig:diffusion-colors} in the paper and compute the corresponding ideal words, that we write as $\vu_{\rm color' \, object'} \approx \vu_{0}' + \vu_{\rm color'}' + \vu_{\rm obj'}'$. We then investigate whether families of ideal words computed independently can be ``mixed,'' combining ideal words for colors from the first collection and ideal words for objects from the second one, and vice-versa. Figure~\ref{fig:iw-transfer} shows that this is possible, at least in our restricted setting. In the first row, we show examples of four new objects with different colors computed by adding associated ideal words (\{white, pink, orange, black\} $\times$ \{chair, wallet, shirt, pen\}). In the next two rows, we use the ideal words for objects with the ideal words for colors obtained previously; in the last two rows, we use the ideal words for the new colors together with the ideal words for the objects obtained previously. To obtain all of these images, we simply used $\vu_{\rm color' \, object} \approx (\vu_0 + \vu_0')/2 + \vu_{\rm color'}' + 2\cdot \vu_{\rm obj}$ (we found that amplifying the ideal words for objects helps ensure that objects are more centered). Analyzing the limits of this sort of transferability is left for future work.

Finally, in Figure~\ref{fig:diffusion-context} we generate images with ideal words while also using a third ``context'' factor, in addition to the ones corresponding to color and object (for those we use the same colors and objects as in Figure~\ref{fig:diffusion-colors}). Here we see that linear compositionality is effective using simple contexts such as \{on the beach, on a street\} (first two rows), however using more complex contexts such as \{underwater, in a volcano\} (third and fourth row) it fails to produce good results.

\begin{figure}[t]
    \centering
    \includegraphics[width=.95\columnwidth]{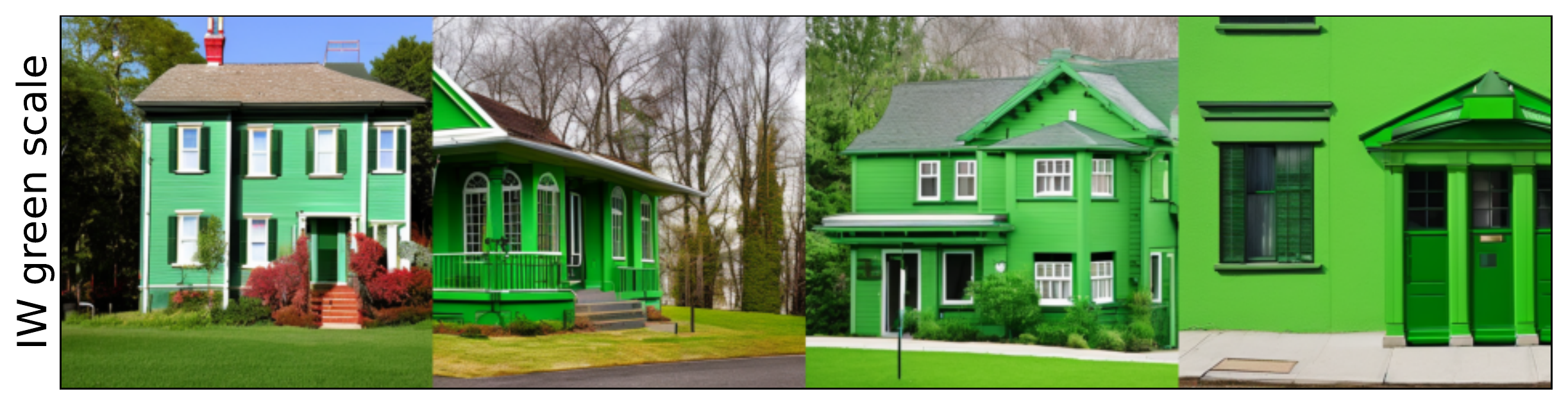}\\
    \includegraphics[width=.95\columnwidth]{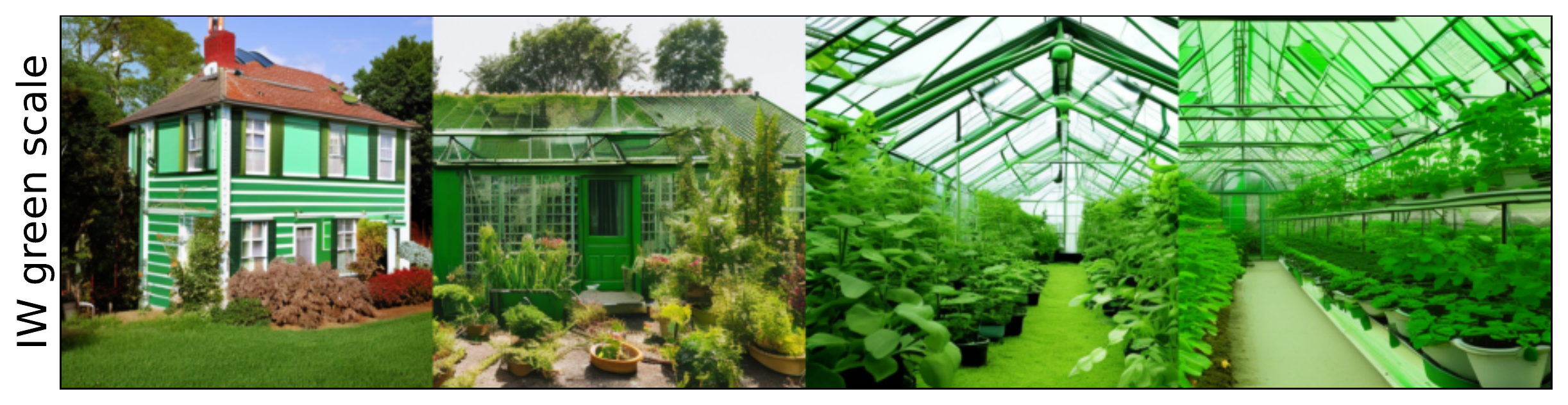}\\
    \caption{Scaling the ideal word $\vu_{\rm green}$ a by factor $\gamma = .5, 1, 1.5, 2$, respectively. \emph{Top:} $\vu_{\rm green}$ is computed using all objects as contexts. \emph{Bottom:} $\vu_{\rm green}$ is computed only ``house'' as context.}
    \label{fig:green-amplify}
\end{figure}

\begin{figure}[h]
    \centering
    \includegraphics[width=.85\columnwidth]{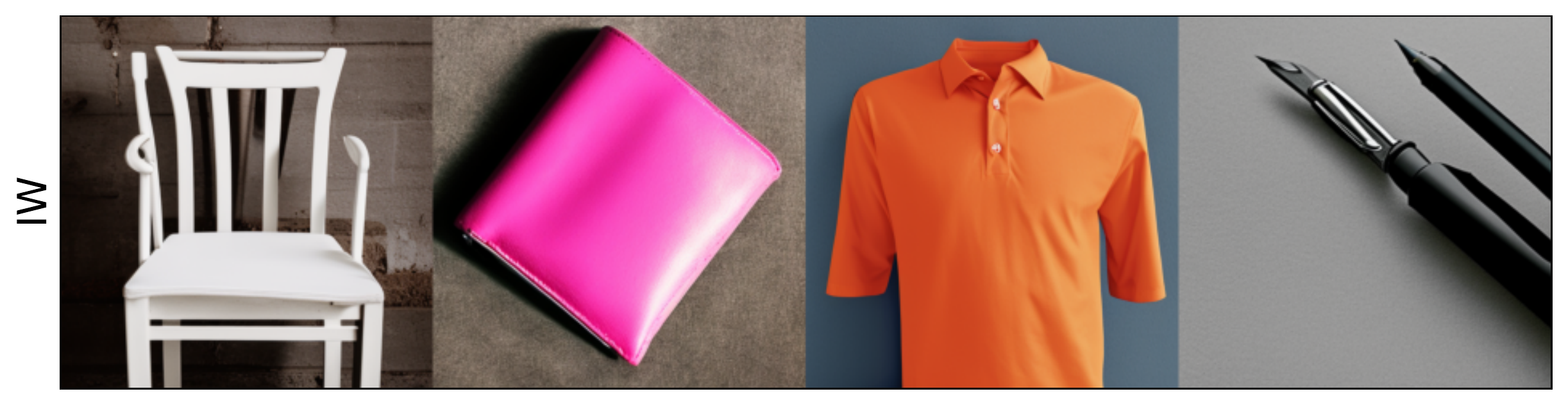}\\
    \includegraphics[width=.85\columnwidth]{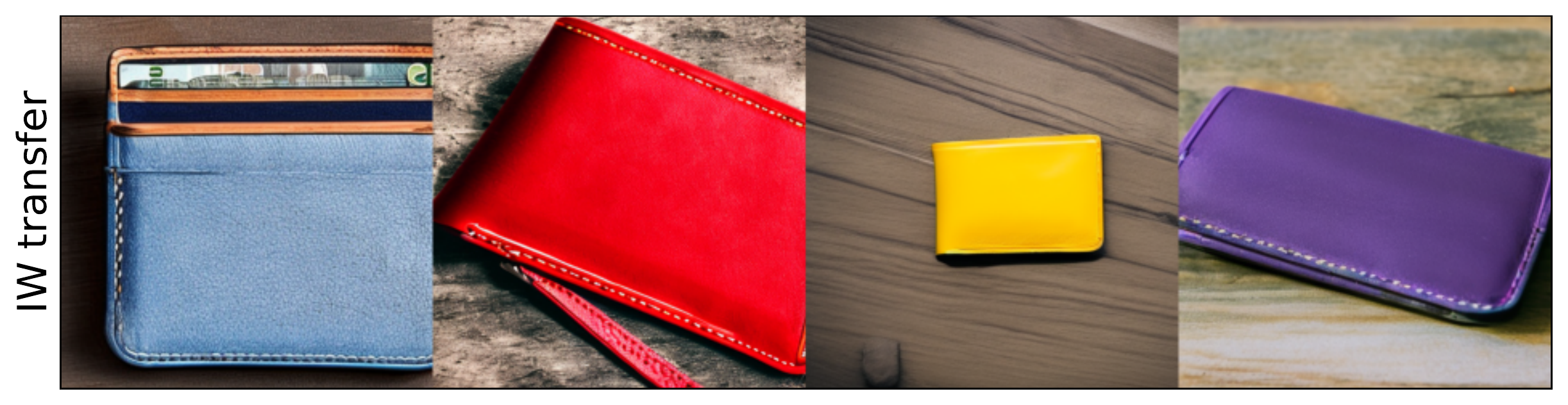}\\
    \includegraphics[width=.85\columnwidth]{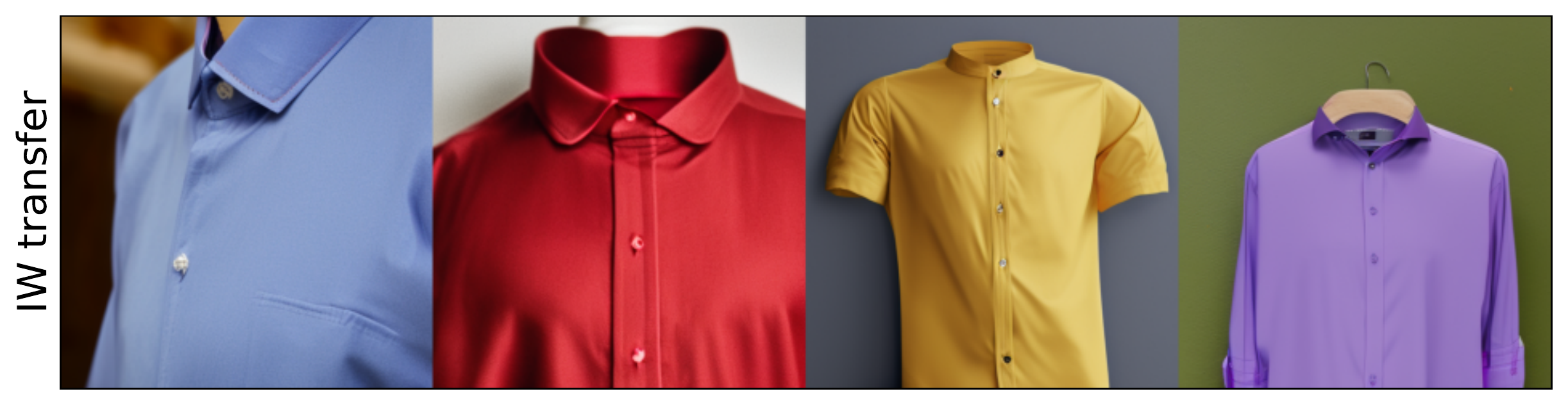}\\
    \includegraphics[width=.85\columnwidth]{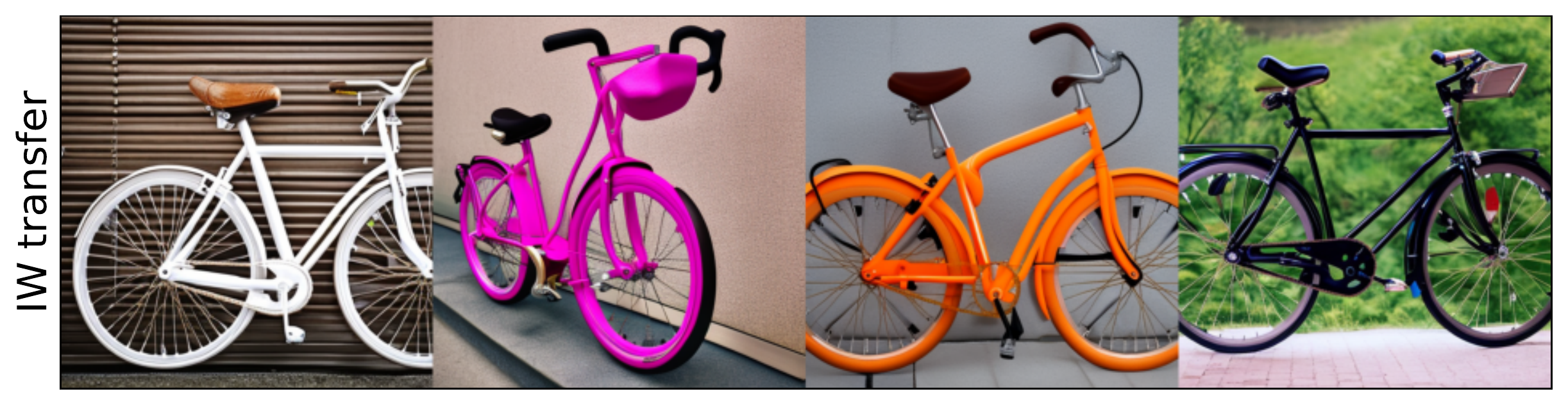}\\
    \includegraphics[width=.85\columnwidth]{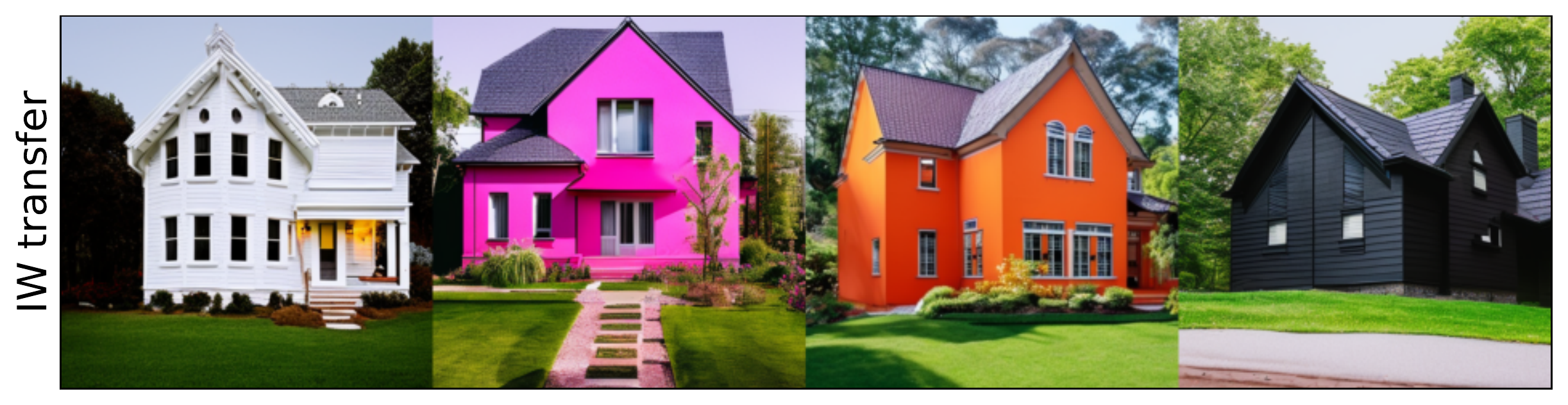}\\
    \caption{Transferring ideal words. \emph{Top row:} Images generated ideal words for a different set of colors and objects compared to the ones used Figure~\ref{fig:diffusion-colors}. \emph{Second and third rows:} 
    images generated by adding new ideal words for objects with the previous ideal words for colors; \emph{Fourth and fifth rows:} images generated by adding new ideal words for colors with the previous ideal words for objects.
    }
    \label{fig:iw-transfer}
\end{figure}
\begin{figure}[h]
    \centering
    \includegraphics[width=.85\columnwidth]{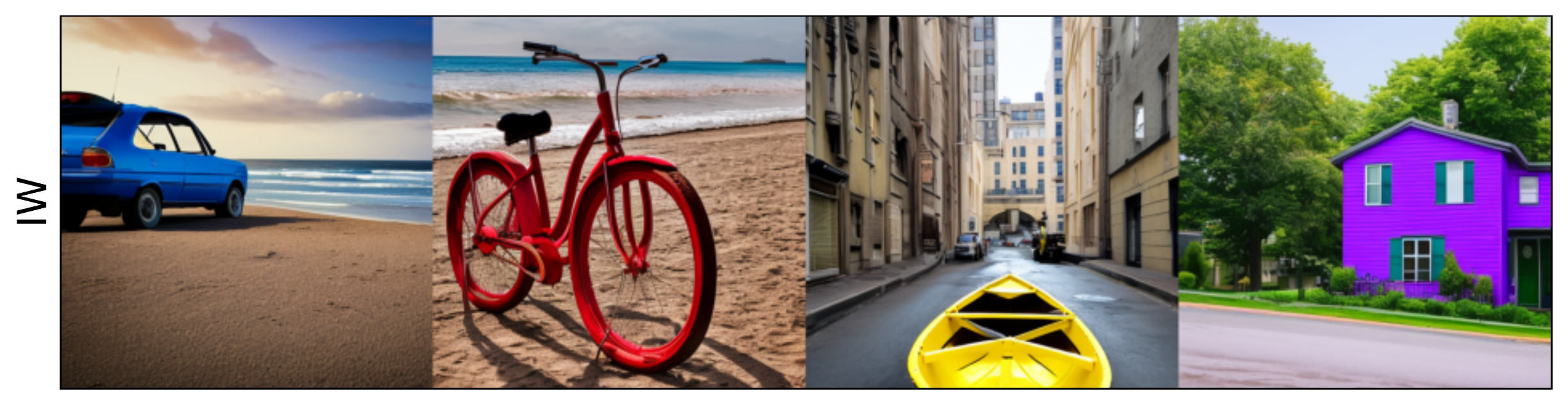}\\
    \includegraphics[width=.85\columnwidth]{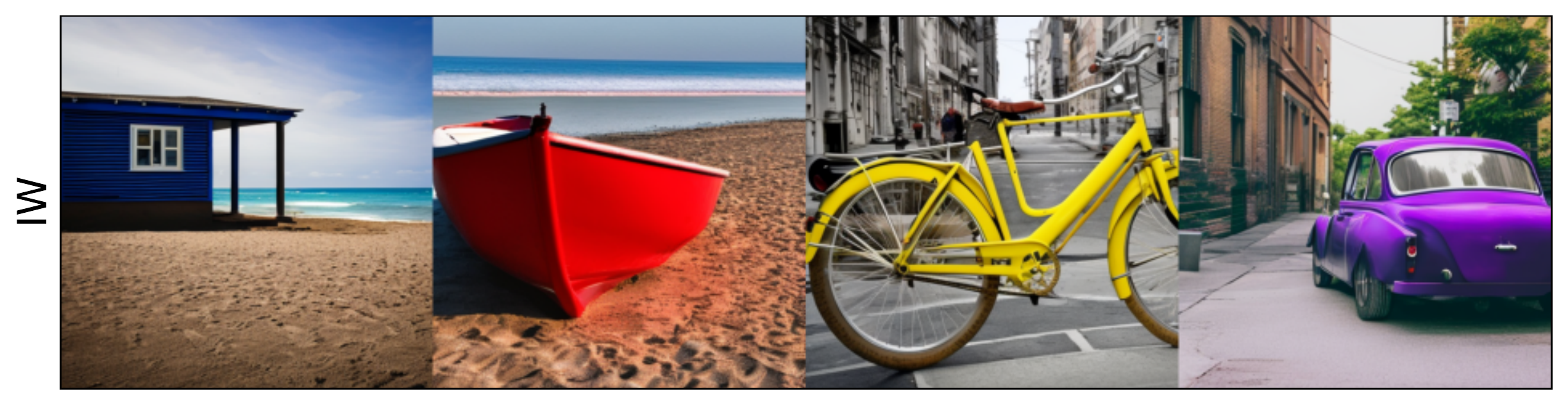}\\
    \includegraphics[width=.85\columnwidth]{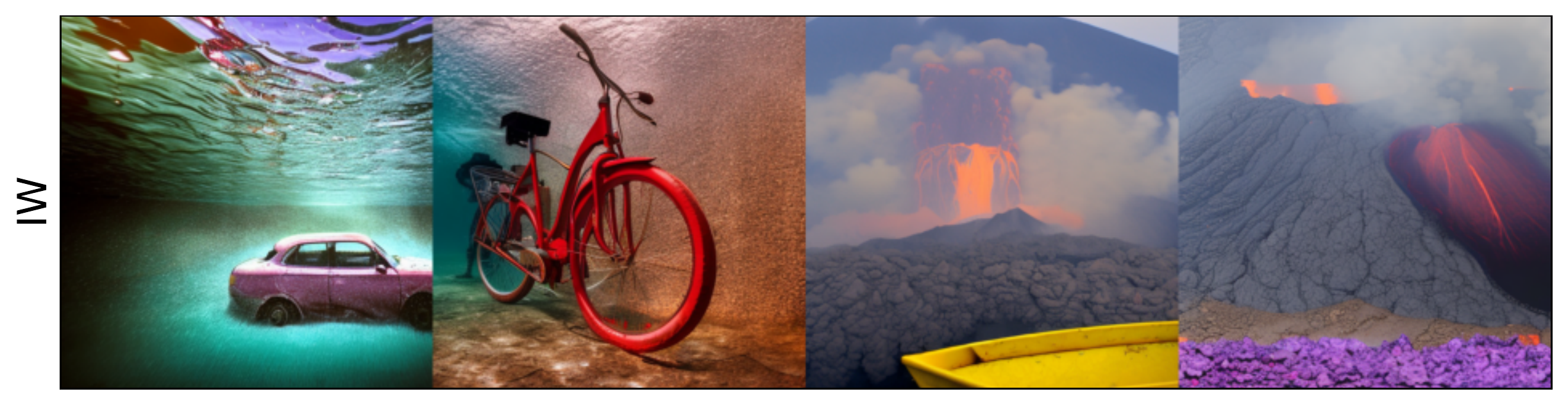}\\
    \includegraphics[width=.85\columnwidth]{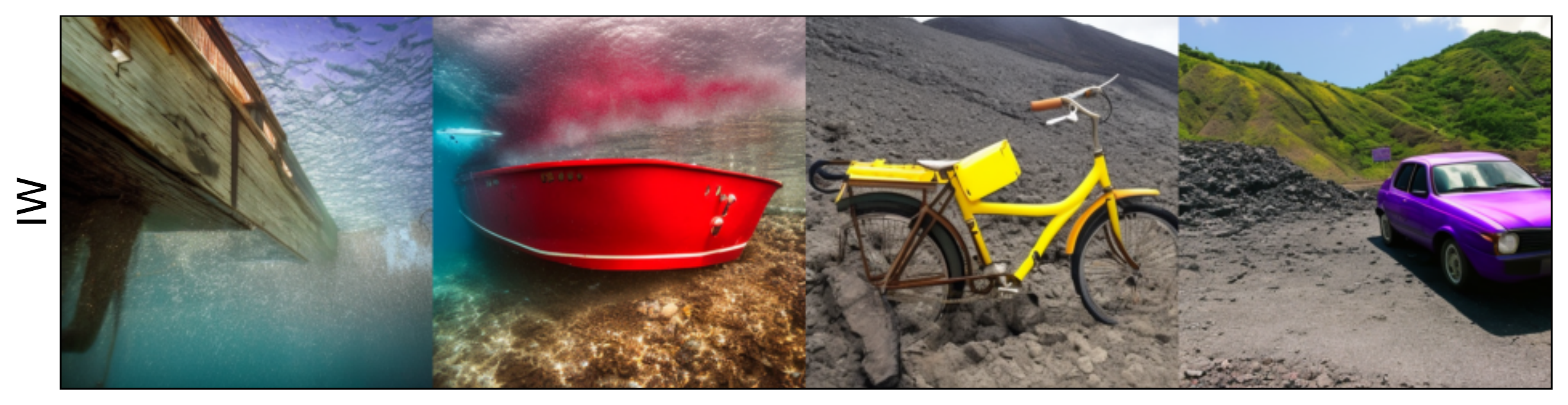}\\
    \caption{Images generated using ideal words with t8ree factors: color, object, context. \emph{First two rows:} using context factor \{on the beach, on a street\}; \emph{Second two rows:} using context factor \{underwater, in a volcano\}.}
    \label{fig:diffusion-context}
\end{figure}

\end{document}

%% file: math_commands.tex
\usepackage{amsmath,amsfonts,bm}

\def\eqref#1{equation~\ref{#1}}
\def\1{\bm{1}}

\def\ve{{\bm{e}}}

\def\vp{{\bm{p}}}
\def\vq{{\bm{q}}}
\def\vr{{\bm{r}}}
\def\vs{{\bm{s}}}

\def\vu{{\bm{u}}}
\def\vv{{\bm{v}}}
\def\vw{{\bm{w}}}

\DeclareMathAlphabet{\mathsfit}{\encodingdefault}{\sfdefault}{m}{sl}
\SetMathAlphabet{\mathsfit}{bold}{\encodingdefault}{\sfdefault}{bx}{n}

\newcommand{\RR}{\mathbb{R}}

\DeclareMathOperator*{\argmax}{arg\,max}

%% file: teaser_tikz.tex
\tdplotsetmaincoords{82}{108} % Set the viewing angle
\begin{tikzpicture}[scale=.80, opacity=1]

\begin{scope}[scale=1.2,xshift=0.1cm, yshift=-.1cm, tdplot_main_coords,font=\tiny]
\coordinate (A) at (0,0,0);
\coordinate (B) at (0,.9,1.53);
\coordinate (C) at (0,1.8,0);
\coordinate (D) at (3.5,0,0);
\coordinate (E) at (3.5,1,1.7);
\coordinate (F) at (3.5,2,0);

% Calculate the barycenter
\coordinate (DEF) at (barycentric cs:D=1,E=2,F=1);
\coordinate (ABC) at (barycentric cs:A=1,B=2,C=1);

\coordinate (BE) at (barycentric cs:B=1,E=1);
\coordinate (CF) at (barycentric cs:C=1,F=1);
\coordinate (AD) at (barycentric cs:A=1,D=1);

\draw[dashed] (A) -- (B);
\draw[] (B) -- (C);
\draw[dashed] (C) -- (A);

\draw (D) -- (E);
\draw (E) -- (F);
\draw (F) -- (D);

\draw[dashed] (A) -- (D);
\draw (B) -- (E);
\draw (C) -- (F);

\fill[fill=red!20,opacity=0.4] (D) -- (E) -- (F) -- cycle;
\fill[fill=green!20,opacity=0.4] (F) -- (E) -- (B) -- (C) -- cycle;

\foreach \vertex in {A,B,C,D,E,F}
    \filldraw [black!60, opacity=1] (\vertex) circle (.03cm);

\node[anchor=north, black!60] at (A) {a red car};
\node[anchor=south east, black!60] at (B) {a green car};
\node[anchor=north west, black!60] at (C) {a blue car};
\node[anchor=north, black!60] at (D) {a red bike};
\node[anchor=south east, black!60] at (E) {a green bike};
\node[anchor=north west, black!60] at (F) {a blue bike};

\end{scope}

\begin{scope}[scale=1.6, xshift=2.9cm, yshift=-1cm, tdplot_main_coords, font=\tiny, transform shape=false]

\coordinate (000) at (0,0,0);
\coordinate (001) at (0,0,1.5);
\coordinate (010) at (0,1,0);
\coordinate (011) at (0,1,1.5);
\coordinate (100) at (2,0,0);
\coordinate (101) at (2,0,1.5);
\coordinate (110) at (2,1,0);
\coordinate (111) at (2,1,1.5);

\path[fill=blue!20,opacity=0.4] (100) -- (101) -- (111) -- (110) -- cycle;
\path[fill=red!20,opacity=0.4] (110) -- (111) -- (011) -- (010) -- cycle;
\path[fill=green!20,opacity=0.4] (111) -- (101) -- (001) -- (011) -- cycle;

\draw[dashed] (010) -- (000) -- (001);
\draw (001) -- (011) -- (010);
\draw (100) -- (101) -- (111) -- (110) -- (100);

\draw[dashed] (000) -- (100);
\draw (001) -- (101);
\draw (011) -- (111);
\draw (010) -- (110);

\foreach \vertex in {000,001,010,011,100,101,110,111}
    \filldraw [black!60, opacity=1] (\vertex) circle (.03cm);

\node[align=center, anchor=south, black!60] at (000) {a cold\\ sunny day};
\node[align=center, anchor=south, black!60] at (001) {a cold\\ rainy day};
\node[align=center, anchor=south, black!60] at (010) {a cold\\ sunny night};
\node[align=center, anchor=south, black!60] at (011) {a cold\\ rainy night};
\node[align=center, anchor=south, black!60] at (100) {a warm\\ sunny day};
\node[align=center, anchor=south, black!60] at (101) {a warm\\ rainy day};
\node[align=center, anchor=south, black!60] at (110) {a warm\\ sunny night};
\node[align=center, anchor=south, black!60] at (111) {a warm\\ rainy night};

\end{scope}

\draw[black!60] (2.5,0) ellipse (4.6cm and 3cm);

\node [draw,
        black!60, 
        thick, 
        fill=white, 
        text centered, 
        text width=3.5cm, 
        minimum height=.5cm,
        rounded corners=5pt
        ] at (4.0,2.3) {Embedding Space};

\coordinate (vec) at (-.8,-2.4);
\draw[thick,black!60,->,shorten >=.7cm] (0,-1) -- (vec);
\node [
        black!60, 
        thin,  
        fill=white, 
        draw,
        text centered, 
        text width=3cm,
        minimum height=.5cm, 
        rounded corners=5pt,
        font=\footnotesize,
        ] at (vec) {$\bm u(\mbox{``\texttt{a [col] [obj]}'')}$\\$\,\approx \bm u_0 + \bm u_{\rm col} + \bm u_{\rm obj}$};

\end{tikzpicture}